\documentclass{article}

\usepackage{hyperref}

\usepackage[accepted]{icml2024}

\usepackage{amsmath}
\usepackage{amssymb}
\usepackage{mathtools}
\usepackage{amsthm}

\usepackage[capitalize,noabbrev]{cleveref}

\theoremstyle{plain}
\newtheorem{theorem}{Theorem}[section]
\newtheorem{proposition}[theorem]{Proposition}
\newtheorem{lemma}[theorem]{Lemma}

\theoremstyle{definition}
\newtheorem{definition}[theorem]{Definition}

\theoremstyle{remark}

\usepackage{hyperref}
\usepackage{url}
\usepackage{bbm}
\usepackage{bm}
\usepackage{booktabs}

\usepackage{colortbl}
\usepackage[inkscapeformat=png]{svg}
\definecolor{LightGreen}{rgb}{0.93,0.98,0.96}
\definecolor{Green}{rgb}{0.0,0.5,0.0}

\newcommand{\ie}{\textit{i}.\textit{e},\ }

\usepackage{microtype}
\usepackage{graphicx}
\usepackage{subfigure}
\usepackage{subcaption}
\usepackage{siunitx}

\DeclareMathOperator*{\esssup}{ess\,sup}
\DeclareMathOperator*{\essinf}{ess\,inf}

\usepackage[textsize=tiny]{todonotes}

\icmltitlerunning{Unraveling the Impact of Heterophilic Structures on Graph Positive-Unlabeled Learning}

\begin{document}

\twocolumn[

\icmltitle{Unraveling the Impact of Heterophilic Structures on Graph \\ Positive-Unlabeled Learning}

\begin{icmlauthorlist}
\icmlauthor{Yuhao Wu}{usyd}
\icmlauthor{Jiangchao Yao}{SJTU,SJlab}
\icmlauthor{Bo Han}{hkbu}
\icmlauthor{Lina Yao}{unsw}
\icmlauthor{Tongliang Liu}{usyd}
\end{icmlauthorlist}

\icmlaffiliation{usyd}{Sydney AI Center, The University of Sydney}
\icmlaffiliation{SJTU}{CMIC, Shanghai Jiao Tong University}
\icmlaffiliation{SJlab}{Shanghai AI Laboratory}
\icmlaffiliation{unsw}{University of New South Wales}
\icmlaffiliation{hkbu}{TMLR Group, Department of Computer Science, Hong Kong Baptist University}
\icmlcorrespondingauthor{Tongliang Liu}{tongliang.liu@sydney.edu.au}

\icmlkeywords{Machine Learning, ICML}

\vskip 0.3in
]



\printAffiliationsAndNotice{}  

\begin{abstract}

While Positive-Unlabeled (PU) learning is vital in many real-world scenarios, its application to graph data still remains under-explored. We unveil that a critical challenge for PU learning on graph lies on the edge heterophily, which directly violates the \textit{irreducibility
assumption} for \textit{Class-Prior Estimation} (class prior is essential for building PU learning algorithms) and degenerates the latent label inference on unlabeled nodes during classifier training. In response to this challenge, we introduce a new method, named \textit{\underline{G}raph \underline{P}U Learning with \underline{L}abel Propagation Loss} (GPL). Specifically, GPL considers learning from PU nodes along with an intermediate heterophily reduction, which helps mitigate the negative impact of the heterophilic structure.
We formulate this procedure as a bilevel optimization that reduces heterophily in the inner loop and efficiently learns a classifier in the outer loop. Extensive experiments across a variety of datasets have shown that GPL significantly outperforms baseline methods, confirming its effectiveness and superiority.

\end{abstract}

\section{Introduction}

Positive-Unlabeled (PU) learning~\citep{denis1998pac,de1999positive,denis2005learning}, where a binary classifier is trained from positive and unlabeled samples, has made remarkable progress on image~\citep{kiryo2017positive,niu2016theoretical,garg2021mixture} and text~\citep{li2016classifying} data. Nowadays, there is burgeoning interest in extending PU learning to graph-structured data due to the ubiquitous nature of real-world contexts~\citep{hu2020open,gaudelet2021utilizing,velivckovic2023everything,wu2024mitigating}. For instance, in online transaction networks, fraudsters are labeled as positive nodes only upon being detected~\citep{yoo2021accurate}, and a similar case holds in pandemic graph prediction~\citep{panagopoulos2021transfer} where only affected nodes are identified as positive, while the remaining nodes are actually label-agnostic.

However, most current approaches to PU learning are built on the hypothesis that each sample is independently generated. Such a premise hinders these methods from readily adapting to graph data where nodes are interdependent by edges~\cite{wu2022handling,zhao2020uncertainty}. Our investigation reveals that a primary challenge in adapting these approaches to graph data is the prevalence of \textit{heterophilic structure}, where positive and negative nodes are connected~\citep{zhu2020beyond,zhang2019heterogeneous}. Such heterophily  widely existed in standard graph learning~\citep{platonov2023characterizing,luan2022revisiting}, intensifies when dealing with only positive and unlabeled nodes of interest.

The heterophilic structure impedes two aspects of PU learning: (i) It results in the latent feature entanglement of positive and negative nodes due to heterophilic edges, which directly challenges \textit{Class-Prior Estimation} (CPE), \ie estimating the fraction of positive nodes among the unlabeled nodes. CPE methods rely on the \textit{irreducibility} assumption, which holds that the patterns belonging to positive nodes could not possibly be confused with patterns from negative nodes.~\citep{garg2021mixture,bekker2020learning,ivanov2020dedpul}. However, heterophilic structures violate this assumption, yielding an overestimated class prior. (ii) The presence of heterophilic structure complicates the latent label inference for unlabeled nodes. More precisely, the heterophilic edges between different classes mislead the message passing among same-class nodes, hindering the accurate inference of latent labels for unlabeled nodes during the classifier training process (termed as \textit{PU classification} for simplicity)~\cite{garg2021mixture,yao2020rethinking}. The adverse effect to CPE and PU classification significantly challenges graph PU learning when heterophilic structure exists.

In our paper, we present a \textit{\underline{G}raph \underline{P}U Learning with \underline{L}abel Propagation Loss} (GPL) method to alleviate the side effect of heterophilic structures on PU learning. GPL mitigates the impact of heterophilic edges by reducing their weights based on a proposed Label Propagation Loss (LPL). Specifically, optimizing LPL results in strengthening the weights of homophilic edges, \ie edges linking same-class nodes, while simultaneously reducing the weights of heterophilic edges, \emph{relying solely on the observed positive nodes}. This approach is effective because homophilic edges inherently support the label propagation process, whereas heterophilic edges between different classes are counterproductive. Therefore, by matching the given/predicted labels and the propagated labels, we can reduce the impact of heterophilic edges, which then promotes the accuracy of CPE and PU classification.

Taking LPL and the vanilla graph PU learning interactively forms the basic idea of GPL: the diminished heterophilic structure by LPL creates a conducive context for graph PU learning. The classifier trained through graph PU learning, in turn, helps refine this structure with the aid of more accurate prediction labels. We formulate the optimization of LPL and graph PU learning as an efficient bilevel optimization, where the heterophily is reduced in the inner loop and the classifier is learnt in the outer loop. In a nutshell, our contributions can be summarized into the following points:
\begin{itemize}
    \item We first identify the challenge of learning from PU nodes with heterophilic structures. Through both theoretical analysis and empirical evidence, we formally elucidate the difficulty encountered in CPE and PU classification when the heterophilic structure exists.
    \item We propose an end-to-end GPL method, which considers the optimization of the graph structure to reduce the negative impact of heterophilic edges on graph PU learning. This reduction has been theoretically proved. Besides, GPL leverages a bilevel algorithm, obviating the need for a known class-prior and eliminating the assumption of a homophily graph structure.
    \item We conduct extensive experiments on datasets ranging from homophily to heterophily to demonstrate the superiority of our proposed framework over the baselines. Furthermore, we provide comprehensive analyses of the underlying mechanisms of our framework.
\end{itemize}

\section{Preliminaries}
\subsection{Notation}

We consider a set of positive (P) and negative (N) nodes with indices $i \in \{1,2,\cdots,n\}$, represented by an observed undirected graph $ G = (\mathcal{V}, \mathcal{E})$, where $\mathcal{V} = \{v_1, \dots, v_n\}$ denotes the node set and $\mathcal{E} = \{ e_{ij}\}$ denotes the edge set. The observed edges induce an adjacency matrix $\bm{A} \in \mathbb{R}^{n \times n}$ where $\bm{A}_{ij} = 1$ if nodes $i$ and $j$ are connected by edges  $e_{ij}$ and $\bm{A}_{ij} = 0$ otherwise. Moreover, each node $v_i$ has an input feature vector denoted by $\bm{x}_{i} \in \mathbb{R}^{D}$ where $D$ is the dimension, and a latent binary label $y_{i} \in \{-1, +1\}$. Here, assume the P and N nodes are drawn from the P and N class-conditional density $[\mathbbm{P}_{\mathrm{p}}(\bm{x},\mathcal{G}_{\bm{x}}^{\bm{A}}),\mathbbm{P}_{\mathrm{n}}(\bm{x},\mathcal{G}_{\bm{x}}^{\bm{A}})] := [\mathbbm{P}(\bm{x},\mathcal{G}_{\bm{x}}^{\bm{A}}|y=+1), \mathbbm{P}(\bm{x},\mathcal{G}_{\bm{x}}^{\bm{A}}|y=-1)]$, where $\mathcal{G}_{\bm{x}}^{\bm{A}} = ( \{\bm{x}_j\}_{j\in\mathcal{N}_{\bm{x}}},\bm{A}_{\bm{x}} )$ denotes the ego-graph centered at node $\bm{x}$, $\mathcal{N}_{\bm{x}}$ is the set of neighboring nodes (within a certain number of hops) in the ego-graph  of $\bm{x}$ and $\bm{A}_{\bm{x}}$ is its associated adjacency matrix. In learning from positive and unlabeled (U) nodes, we only know a fraction of P nodes, denoted as a set $\mathcal{P} \sim \mathbbm{P}_\mathrm{p}(\bm{x},\mathcal{G}_{\bm{x}}^{\bm{A}})$. The other $\mathcal{V} \backslash \mathcal{P}$ is unlabeled and represented as a set $\mathcal{U} \sim \mathbbm{P}(\bm{x},\mathcal{G}_{\bm{x}}^{\bm{A}})$. Given $G$ with $\mathcal{P}$ and $\mathcal{U}$, our goal is to learn a binary classifier $f_{w}:(\bm{x},\mathcal{G}_{\bm{x}}^{\bm{A}})  \rightarrow [0, 1]$ that accurately approximates $\mathbbm{P}(y = +1 | \bm{x},\mathcal{G}_{\bm{x}}^{\bm{A}}) $ on $\mathcal{U}$ with the trainable parameters $w$. Our formulation is fairly different from existing PU learning works that assume \textit{i.i.d.} inputs.  

\begin{figure} [t]
\centering
\includegraphics[width=1\linewidth]{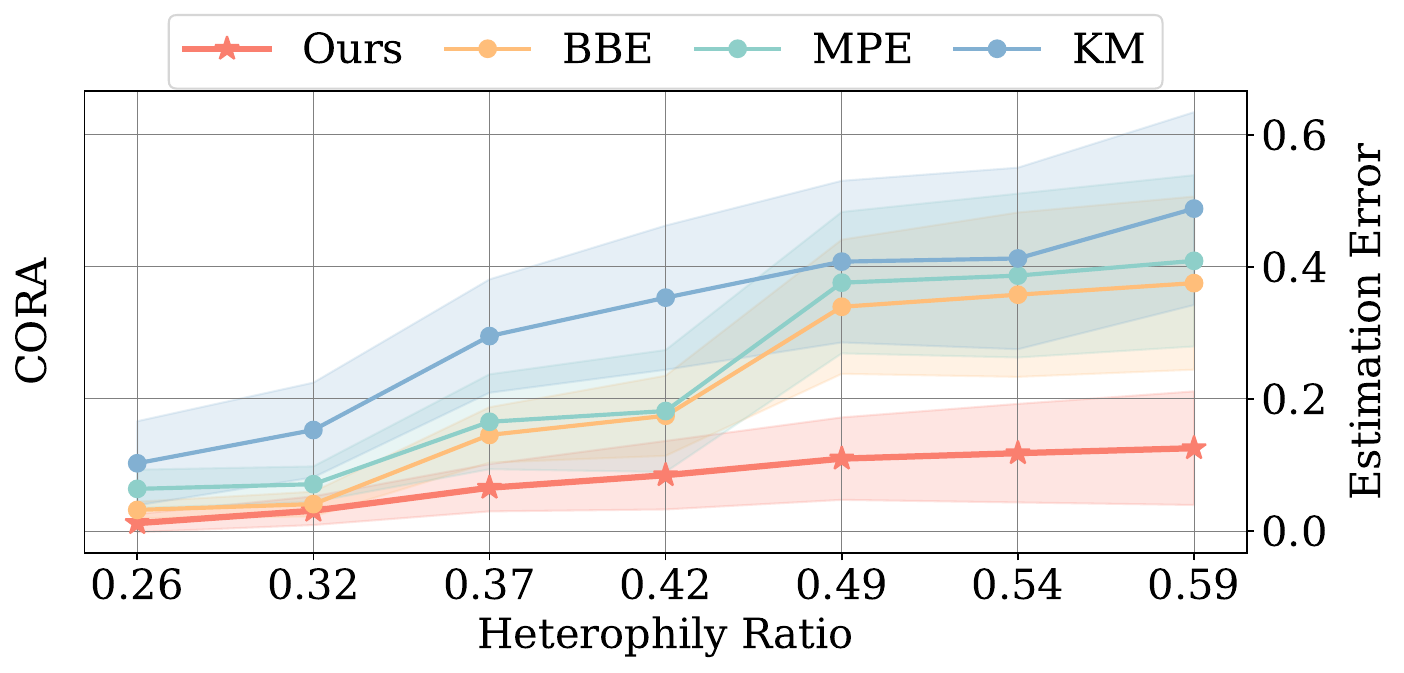}
\vspace{-10pt}
\caption{ Estimation error of our method and the baseline estimator for Class-Prior Estimation, under varying heterophily ratio. The way to modify the graph structure with different heterophily ratios is referred to~\citep{ma2021homophily}.}
\label{figure::diff_ratio_cpe}
\vspace{-5pt}
\end{figure}

\subsection{Motivation}

\paragraph{Class-Prior Estimation (CPE) with Heterophily}

Formally, CPE plays the role of estimating the class prior $\pi_\mathrm{p}\in (0,1)$ for PU learning methods~\citep{yao2020rethinking,garg2021mixture}, given samples from the marginal distribution $\mathbbm{P}(\bm{x},\mathcal{G}_{\bm{x}}^{\bm{A}})$ and samples from class-conditional distribution $\mathbbm{P}_\mathrm{p}(\bm{x},\mathcal{G}_{\bm{x}}^{\bm{A}})$ of P nodes. Specifically, $\mathbbm{P}(\bm{x},\mathcal{G}_{\bm{x}}^{\bm{A}})$ is mixture of class-conditional distributions of P and N nodes, \ie 
\begin{equation}
\mathbbm{P}(\bm{x},\mathcal{G}_{\bm{x}}^{\bm{A}})=\pi_\mathrm{p}\mathbbm{P}_\mathrm{p}(\bm{x},\mathcal{G}_{\bm{x}}^{\bm{A}}) + (1-\pi_\mathrm{p})\mathbbm{P}_\mathrm{n}(\bm{x},\mathcal{G}_{\bm{x}}^{\bm{A}}).
\label{ptr}
\end{equation}
To ensure the identifiability of $\pi_\mathrm{p}$, almost all CPE algorithms assume $\mathbbm{P}_\mathrm{n}$ to be irreducible with respect to $\mathbbm{P}_\mathrm{p}$~\citep{blanchard2010semi,ivanov2020dedpul,jain2016nonparametric,wu2023making,liu2015classification}. However, this assumption is based on \textit{independent and identically distributed} (i.i.d) data, which is not applicable to graph data. Here, we use the following theorem to specify the form of irreducible condition in our graph PU learning setting and leave its complete proof in Appendix~\ref{app:A.} for reference.

\begin{figure*} [h]
\centering
\includegraphics[width=1\linewidth]{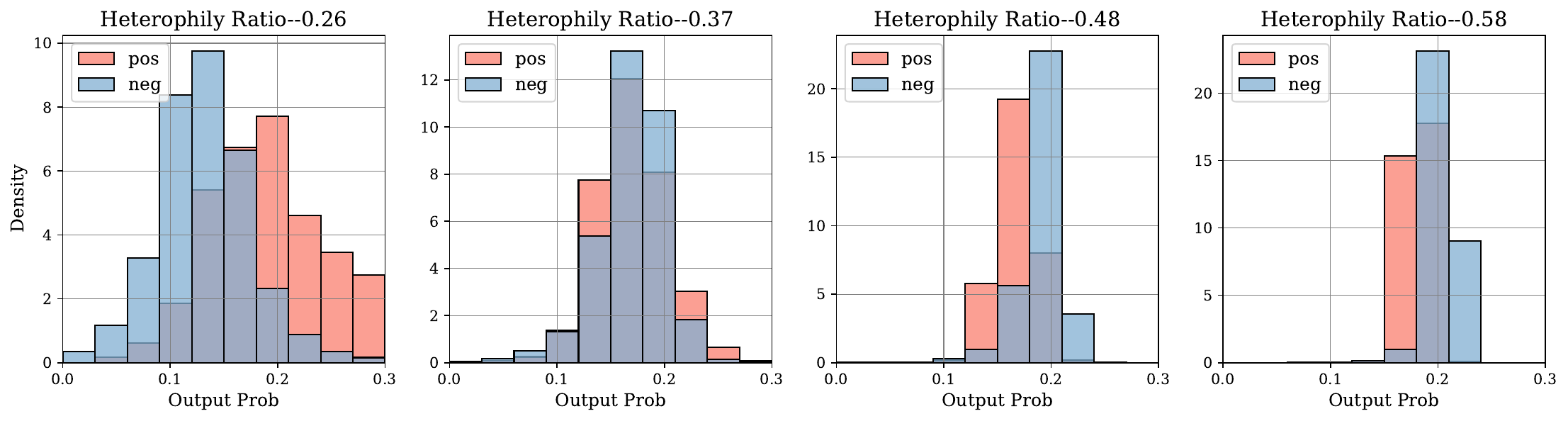}

\caption{The distributions of the predicted probabilities (of being positive) for unlabeled nodes when training a classifier on positive and unlabeled CORA dataset, under varying heterophily ratios. The methodology employed to modify the graph structure with different heterophily ratios can be referred to~\citep{ma2021homophily}. We show that as the heterophily ratio increases, the distribution of predicted probabilities of positives and negatives become less and less separable.}

\label{figure::pre_distribution_hete}
\end{figure*}

\begin{theorem} [Graph Irreducible Condition] 
\label{irr_condtion}
In a graph, $\mathbbm{P}_\mathrm{n}$ is irreducible with respect to $\mathbbm{P}_\mathrm{p}$ iff it satisfies the following condition: 
\begin{equation}
    \esssup_{i \in \mathcal{V}} \mathbbm{P}(y = +1 | \bm{x}_i,  \mathcal{G}_{\bm{x}_i}^{\bm{A}})  = 1,
\end{equation}
where $\esssup$ is the essential supremum.
\end{theorem}

With theorem~\ref{irr_condtion}, the irreducible condition on a graph essentially assumes that there exists one subset of nodes, along with their connected neighbouring nodes from $\mathcal{G}_{\bm{x}_i}^{\bm{A}}$, must all belong to the P class with probability one. However, the heterophilic structures in the graph suggest that P nodes are likely to have N nodes as neighbours, and vice versa. This characteristic violates the irreducible condition of the graph, yielding an overestimated class prior. As demonstrated in Figure~\ref{figure::diff_ratio_cpe}, we can observe that the estimation error for the class prior escalates with the increase in heterophilic structure, which thus recalls extra design for graph PU learning.

\paragraph{PU classification with Heterophily}
The typical approach in PU classification is to distinguish P and N data from U data, followed by training a classifier on estimated P and N data~\citep{garg2021mixture,xinrui2023beyond}. Ideally, with $\pi_\mathrm{p}$, if we can distinguish all P and N data from the U data, then we can expect to achieve a well-performed classifier, akin to one trained using fully-supervised data. The core of graph PU classification hinges on the distinguishing performance from U nodes, which is complicated in the presence of heterophilic structures. As in Figure~\ref{figure::pre_distribution_hete}, with an increasing heterophilic ratio, the embeddings of P and N nodes from the GNN classifier become more indistinct, making the identification more challenging within the U nodes.

\section{Methodology}

Basically, a straightforward application of PU learning on graph data can be formulated as follows:
\begin{equation}  
\label{two_stage_cpe}
\begin{split}
& \text{Stage 1:} \quad   \hat{\pi}_\mathrm{p} = \mathcal{E}_{\textit{CPE}}(\bm{A}),\\
& \text{Stage 2:} \quad    w^* := \arg\min_{w} \mathcal{L}_{\textit{GNN}}(\hat{\pi}_\mathrm{p}, w,\bm{A}).
\end{split}
\end{equation}
Here, $w^*$ represents the optimal parameter for a GNN classifier. This parameter is derived by minimizing the loss using the pre-estimated class prior $\hat{\pi}_\mathrm{p}$, from an estimator $\mathcal{E}_{\textit{CPE}}(\bm{A})$ with the adjacent matrix $\bm{A}$. As previously discussed, this process is intrinsically challenged by heterophilic structures characterized in the adjacency matrix $\bm{A}$. To address this, we propose the \textit{\underline{G}raph \underline{P}U Learning with \underline{L}abel Propagation Loss} (GPL) method, which different from the two-stage form in Equation~\ref{two_stage_cpe}, has the following formulation:
\begin{equation}  
\label{bi_op}
\begin{split}
    & \min_{w} \mathcal{L}_{\textit{GNN}}(\hat{\pi}_\mathrm{p}, w,\hat{\bm{A}}) \\
    \text{s.t.} \quad &  \hat{\pi}_\mathrm{p} = \mathcal{E}_{\textit{CPE}}(\hat{\bm{A}})\\
    & \hat{\bm{A}} = \arg\min \limits_{\bm{A}} \mathcal{L}_{\textit{LPL}}(\bm{A})^{(K)},\\
\end{split}
\end{equation}
where the heterophilic structure $\bm{A}$ will be optimized to $\hat{\bm{A}}$ by the proposed \textit{Label Propagation Loss} (LPL) in the next subsection. Intuitively, Equation~\ref{bi_op} explores to introduce an intermediate heterophily reduction procedure to the vanilla graph PU learning, which helps reduce the negative impact in the aforementioned sections. In the following, we will discuss the detailed design and the optimization process.  

\subsection{Reducing Heterophily with Label Propagation Loss }

Despite the straightforward idea in Equation~\ref{bi_op}, it is actually challenging in technique to directly reduce the heterophilic structures, since there are no annotated negative nodes available in PU setting. This difficulty hinders the application of previous heterophily reduction methods~\citep{zhang2019heterogeneous,huang2023revisiting}. To address this issue, we incorporate the idea of \textit{label propagation algorithm} (LPA)\footnote{It is a classical algorithm that iteratively propagates estimated labels of nodes across edges within observed graph structures.} that relies solely on observed positive nodes, characterized as follows.

Let $\bm{E}^{(0)} ={[\mathbbm{P}(y = +1 | \bm{x}_i,\mathcal{G}_{\bm{x}_i}^{\bm{A}}),\mathbbm{P}(y = -1 | \bm{x}_i,\mathcal{G}_{\bm{x}_i}^{\bm{A}})]}_{i \in \mathcal{V}}$ be a vector of initial class-posterior probability for nodes in the graph and the updating rule in LPA is 
\begin{equation}
\small
\begin{aligned}
     {\bm{E}}^{(k)} &= \alpha {\bm{E}}^{(k-1)} + ( 1 - \alpha ) {\bm{D}}^{-1} \bm{A} {\bm{E}}^{(k-1)}, \\ 
    {\bm{E}}^{(k)} &= {[{{\mathbbm{P}}(y = +1 | \bm{x}_i,{\mathcal{G}}_{{\bm{x}}_i}^{\bm{A}})}^{(k)},{\mathbbm{P}(y = -1 | {\bm{x}}_i,{\mathcal{G}}_{{\bm{x}}_i}^{\bm{A}})}^{(k)}]}_{i \in \mathcal{V}} \\
    & = [{\bm{E}}_i^{(k)}]_{i \in \mathcal{V}},
\end{aligned}
\end{equation}
where $ 0 < \alpha < 1 $ is a parameter, $k \in K$ is the label propagation iteration times and $\bm{D}$ is the associated diagonal degree matrix of $\bm{A}$. We construct $\bm{M} \odot \bm{A}$, where $\bm{M}$ is a learnable positive mask matrix that represents edge weights. Intuitively, in LPA, homophilic edges generally help positive nodes retain their positive labels, while heterophilic edges can impede these nodes from maintaining their original labels. As a result, the influence through heterophilic edges on a given P node $\bm{x}_a$ is inversely proportional to the probability of $\bm{x}_a$ being positively classified by LPA. This point will be theoretically proved in our Theorem~\ref{heter_pro_LPA}. 

With the above discussion, we can use LPA as a surrogate to reduce the heterophily, which we term as the \textit{Label Propagation Loss} (LPL) as follows,
\begin{equation}
\label{adaptive_edge_op}
\begin{aligned}
\hat{\bm{A}}  &= \arg\min_{\bm{A}} \mathcal{L}_{\textit{LPL}}(\bm{A})^{(K)}\\ 
& = \arg\min_{\bm{A}} \frac{1}{n_{\mathcal{P}}}\sum_{\bm{x}_i \in  \mathcal{P}} \log (\mathbbm{P}(y = -1 | \bm{x}_i,\mathcal{G}_{\bm{x}_i}^{\bm{A}})^{(K)}).
\end{aligned}
\end{equation} 
The LPL technique is designed to minimize the likelihood of P nodes being misclassified as negative by the LPA, which effectively matches the given labels with the labels propagated through the graph structure. Furthermore, the incorporation of predicted labels from a trained model, as detailed in Equation~\ref{adaptive_edge_op_bo}, can further enhance this matching process. 

\subsection{Bilevel Training}

With LPL, we can solve Equation~\ref{bi_op} by bilevel optimization. In the following, we will detail the procedures for CPE and PU classification as well as a bilevel training schedule.

\paragraph{Class-Prior Estimator}
With multiple layers of graph convolution, assume a GNN model transforms each node to a positive posterior probability $\bm{z} \in  [0,1]$, i.e., $\bm{z} = f_{w}(\bm{x},\mathcal{G}_{\bm{x}}^{\bm{A}})$, where $w$ denotes the trainable parameters of the GNN model $f_{w}$. For a probability density function $\mathbbm{P}_{\theta}$ and $f_{w}$, a cumulative distribution function can be defined as 
$
    \mathbbm{Q}^{\bm{A}}_{\theta}(\bm{z}) = \int_{f_{w}(\bm{x},\mathcal{G}_{\bm{x}}^{\bm{A}}) \geq \bm{z}}\mathbbm{P}_{\theta}(\bm{x},\mathcal{G}_{\bm{x}}^{\bm{A}})\text{d}x,
$
which captures the probability that the input $(\bm{x},\mathcal{G}_{\bm{x}}^{\bm{A}})$ is assigned a value larger than $\bm{z}$ by the classifier $f_{w}$ in the transformed space. For each probability density distribution
$\mathbbm{P}_\mathrm{p}(\bm{x},\mathcal{G}_{\bm{x}}^{\bm{A}})$, 
$\mathbbm{P}_\mathrm{n}(\bm{x},\mathcal{G}_{\bm{x}}^{\bm{A}})$ and $\mathbbm{P}_\mathrm{u}(\bm{x},\mathcal{G}_{\bm{x}}^{\bm{A}})$, we can define $\mathbbm{Q}^{\bm{A}}_{\mathrm{p}}(\bm{z})$, $\mathbbm{Q}^{\bm{A}}_{\mathrm{n}}(\bm{z})$ and $\mathbbm{Q}^{\bm{A}}_{\mathrm{u}}(\bm{z})$ respectively. The class prior can be estimated with the optimized graph structure $\hat{\bm{A}}$ (after Equation~\ref{adaptive_edge_op}):
\begin{equation}
\label{cpe}
    \hat{\pi}_\mathrm{p} = \min_{\bm{c} \in [0,1]}\frac{\mathbbm{Q}^{\hat{\bm{A}}}_\mathrm{u}(\bm{c})}{\mathbbm{Q}^{\hat{\bm{A}}}_\mathrm{p}(\bm{c})} = \mathcal{E}_{\textit{CPE}}(\hat{\bm{A}}).
\end{equation}

\begin{algorithm}[t!]
\caption{Algorithm flow of GPL.}
\label{alg:GPL}
\begin{algorithmic} [1]
\STATE { \bfseries Input:} Labeled positive training nodes $\mathcal{P}$ and unlabeled training nodes $\mathcal{U}$, Class prior estimate $\pi_{\mathrm{p}}$, Original adjacent matrix $\bm{A}$
\STATE Initialize a training GNN classifier $f_{w}$ with $\bm{A}$
\STATE $\hat{\bm{A}} \leftarrow \bm{A}$
\FOR{training iteration $t = 1,2\dots E$}
\WHILE{$k < K$ or not converge}
\STATE $\hat{\bm{A}}  = \arg\min_{\hat{\bm{A}}} \mathcal{L}_{\textit{LPL}}(\hat{\bm{A}})^{(k)}$ (Equation~\ref{adaptive_edge_op_bo})
\\ \textcolor{gray}{// i.e., optimizing $\hat{\bm{A}}$ to reduce heterophilic edges}
\ENDWHILE

$\hat{\pi}_{\mathrm{p}} = \mathcal{E}_{\textit{CPE}}(\hat{\bm{A}})$ (Equation~\ref{cpe})
\\\textcolor{gray}{// i.e., estimating ${\pi}_{\mathrm{p}}$ with optimized $\hat{\bm{A}}$}

Minimizing $\mathcal{L}_{\textit{GNN}}(\hat{\pi}_\mathrm{p}, w,\hat{\bm{A}})$ (Equation~\ref{GNN_PU_c})
\\\textcolor{gray}{// i.e., training $f_{w}$ with $\hat{\bm{A}}$ and $\hat{\pi}_{\mathrm{p}}$}
\ENDFOR
\STATE {\bfseries Output:}Trained GNN classifier $f_{w}$ 
\end{algorithmic}
\end{algorithm}

\paragraph{PU Classification}
Given a training set of P nodes $\mathcal{P}$ and U nodes $\mathcal{U}$ with the estimated class prior $\hat{\pi}_\mathrm{p}$, we begin by ranking $\mathcal{U}$ according to their positive class-posterior probability $\mathbbm{P}(y_i = +1 | \bm{x}_i,\mathcal{G}_{\bm{x}_i}^{\bm{A}}), \bm{x}_i \in \mathcal{U}$. Then, in every epoch of training, we define $\mathcal{U} \setminus \mathcal{S}_{\hat{\pi}_\mathrm{p}}$ as a (temporary) set of provisionally N nodes, created by excluding the subset $\mathcal{S}_{\hat{\pi}_\mathrm{p}}$. This subset $\mathcal{S}_{\hat{\pi}_\mathrm{p}}$ consists of U nodes representing the top $\hat{\pi}_\mathrm{p}$ fraction with the highest positive class-posterior probability. Next, we update our GNN classifier by minimising the loss with the $\hat{\bm{A}}$:
\begin{equation}
\small
\begin{aligned}
\label{GNN_PU_c}
\mathcal{L}_{\textit{GNN}}(\hat{\pi}_\mathrm{p}, w,\hat{\bm{A}}) &= \frac{1}{|\mathcal{P} \cup \mathcal{S}_{\hat{\pi}_\mathrm{p}}|}\sum_{\bm{x}_i \in \mathcal{P} \cup \mathcal{S}_{\hat{\pi}_\mathrm{p}} } \mathcal{L}(f_{w}(\bm{x}_i,\mathcal{G}_{\bm{x}_i}^{\hat{\bm{A}}}), +1) \\
&+ \frac{1}{| \mathcal{U} \setminus \mathcal{S}_{\hat{\pi}_\mathrm{p}}|} \sum_{\bm{x}_i \in \mathcal{U} \setminus \mathcal{S}_{\hat{\pi}_\mathrm{p}} } \mathcal{L}(f_{w}(\bm{x}_i,\mathcal{G}_{\bm{x}_i}^{\hat{\bm{A}}}), -1).
\end{aligned}
\end{equation}

Note that in PU classification, the identified P and N nodes are not only utilized to train a binary classifier but also contribute to our heterophily reduction. This dual functionality is possible because the predicted labels of these identified P and N nodes can be matched with the propagated labels, enabling the optimization of the graph structure. This approach parallels the use of original given labels, and thus we enhance Equation~\ref{adaptive_edge_op} with the identified P and N nodes:
\begin{equation}
\small
\label{adaptive_edge_op_bo}
\begin{aligned}
\hat{\bm{A}}  &= \arg\min_{\bm{A}} \mathcal{L}_{\textit{LPL}}(\bm{A})^{(K)}\\ 
& = \arg\min_{\bm{A}} \bigg( \frac{1}{n_{\mathcal{P} \cup \mathcal{S}_{\hat{\pi}_\mathrm{p}}}}\sum_{\bm{x}_i \in  {\mathcal{P} \cup \mathcal{S}_{\hat{\pi}_\mathrm{p}}}} \log \mathbbm{P}(y = -1 | \bm{x}_i,\mathcal{G}_{\bm{x}_i}^{\bm{A}})^{(K)}\\
& +  \frac{1}{n_{ \mathcal{U} \setminus \mathcal{S}_{\hat{\pi}_\mathrm{p}}}}\sum_{\bm{x}_i \in  
 \mathcal{U} \setminus \mathcal{S}_{\hat{\pi}_\mathrm{p}}} \log \mathbbm{P}(y = +1 | \bm{x}_i,\mathcal{G}_{\bm{x}_i}^{\bm{A}})^{(K)} \bigg).
\end{aligned}
\end{equation}

Putting all things together, this entire procedure can be elegantly formulated as a bi-level optimization: In the inner loop, the identified P and N nodes (initially comprising only observed P nodes) are fixed while the graph structure is optimized to maximize label propagation within the same category. In the outer loop, while maintaining a constant optimized graph structure and utilizing CPE, a binary classifier is strategically trained and helps to differentiate N nodes from U nodes. We include the convergence of GPL in Appendix~\ref{app:c.} for completeness, which is an application of standard bilevel optimization. Summary in Algorithm~\ref{alg:GPL}.

\section{Theoretical Analysis}

In this section, we delve into the theoretical properties of GPL. We mainly target to answer the basic question about Equation~\ref{bi_op} in terms of $\mathcal{L}_{\textit{LPL}}$, CPE and PU classification.

\subsection{Why $\mathcal{L}_{\textit{LPL}}$ reduces the heterophilic structure?}

We analyze our LPL with the concept of \textit{influence distirbution}~\citep{xu2018representation,koh2017understanding,chen2022characterizing}. Specifically, we investigate how the output probability of a node $\bm{x}_a$ changes when the initial probability of another node $\bm{x}_b$ is perturbed slightly through $k$ iterations of label propagation. When $\bm{x}_a$ represents a positive node and $\bm{x}_b$ represents a negative node, we assess the influence of $\bm{x}_b$ on $\bm{x}_a$ using label propagation applied to the graph structure. This assessment allows us to quantify the extent of the heterophilic structure between $\bm{x}_b$ and $\bm{x}_a$, signifying the presence of edges connecting positive and negative nodes. We refer to this influence of one class of nodes on another class of nodes as \textit{heterophily influence}(HI), which represents the existence of the heterophilic structure. According to~\citep{koh2017understanding,xu2018representation,wang2020unifying}, the influence can be measured by the gradient of the output probability of $\bm{x}_a$ with respect to the initial probability of $\bm{x}_b$. Thus, we have
\begin{definition} (Heterophily influence with label propagation)
The heterophily influence of a negative nodes $\bm{x}_b$ on positive nodes $\bm{x}_a$ after $k$ iterations of label propagation is:
\begin{equation}
\small
\begin{aligned}
&HI((\bm{x}_a,\mathcal{G}_{\bm{x}_a}^{\bm{A}}),(\bm{x}_b,\mathcal{G}_{\bm{x}_b}^{\bm{A}});k) \\
&= \frac{\partial \left| {{\mathbbm{P}}(y_a = -1 | \bm{x}_a,\mathcal{G}_{\bm{x}_a}^{\bm{A}}})^{(k)}- {\mathbbm{P}(y_a = -1 | \bm{x}_a,\mathcal{G}_{\bm{x}_a}^{\bm{A}}})^{(0)}\right|}{\partial \mathbbm{P}(y_b = -1 | \bm{x}_b,\mathcal{G}_{\bm{x}_b}^{\bm{A}})}.
\end{aligned}
\end{equation}
\end{definition}

According to the definition provided above, we can obtain the total heterophily influence on $\bm{x}_a$ as follows:
\begin{theorem} \label{heter_pro_LPA}
Considering a given positive node $(\bm{x}_a,\mathcal{G}_{\bm{x}_a}^{\bm{A}})$, the total heterophily influence of all other nodes on node $(\bm{x}_a,\mathcal{G}_{\bm{x}_a}^{\bm{A}})$ is proportional to the negatively class-posterior probability of $(\bm{x}_a,\mathcal{G}_{\bm{x}_a}^{\bm{A}})$ by label propagation:
\begin{equation} 
\small
\begin{aligned}
    & \sum_{\bm{x}_i \in \mathcal{V}, \bm{x}_i \neq \bm{x}_a} {HI((\bm{x}_a,\mathcal{G}_{\bm{x}_a}^{\bm{A}}),(\bm{x}_i,\mathcal{G}_{\bm{x}_i}^{\bm{A}});k)} 
    \\
    & := \left| {\mathbbm{P}(y_a = -1 | \bm{x}_a,\mathcal{G}_{\bm{x}_a}^{\bm{A}}})^{(k)} -{\mathbbm{P}(y_a = -1 | \bm{x}_a,\mathcal{G}_{\bm{x}_a}^{\bm{A}}})^{(0)}\right|\\
    & \propto {\mathbbm{P}(y_a = -1 | \bm{x}_a,\mathcal{G}_{\bm{x}_a}^{\bm{A}}})^{(k)},
    \end{aligned}
\end{equation}
where ${\mathbbm{P}(y_a = -1 | \bm{x}_a,\mathcal{G}_{\bm{x}_a}^{\bm{A}})}^{(k)}$ is negative class-posterior probability of $(\bm{x}_a,\mathcal{G}_{\bm{x}_a}^{\bm{A}})$ by $k$-iteration label  propagation.
\end{theorem}

We give a detailed proof of Theorem \ref{heter_pro_LPA} in Appendix \ref{app:B.}. Generally, Theorem~\ref{heter_pro_LPA} indicates that, if the adjacency matrix $\bm{A}$ minimizes the ${\mathbbm{P}(y_a = -1 | \bm{x}_a,\mathcal{G}_{\bm{x}_a}^{\bm{A}})}^{(k)}$ through a $k$-iteration label propagation, they also minimize the heterophily influence on a given positive node. Consequently, we can make the adjacency matrix $\bm{A}$ trainable and learn to minimize the heterophily influence for each observed positive node by adjusting $\bm{A}$, which is equivalent to reducing the heterophilic structures within the graph.

\subsection{Why reducing the heterophilic structure boosts CPE?}

When estimating the CPE, we have
\begin{equation}
    \bm{c}^{*} = \arg\min_{\bm{c} \in [0,1]}\frac{\mathbbm{Q}_\mathrm{u}(\bm{c})}{\mathbbm{Q}_\mathrm{p}(\bm{c})}
\end{equation}
and $\hat{\bm{A}} = \arg\min_{\bm{A}} \mathcal{L}_{\textit{LPL}}(\bm{A})^{(K)}$. We define the estimation error for estimating class prior on the ordinary graph structure $\bm{A}$ as $e^{\bm{A}} = \left| \hat{\pi}_\mathrm{p}^{\bm{A}} - \pi_\mathrm{p}^{\bm{A}}\right|$ and the estimation error for estimating class prior on the optimized graph structure $\hat{\bm{A}}$ as $e^{\hat{\bm{A}}} = \left|\hat{\pi}_\mathrm{p}^{\hat{\bm{A}}} - \pi_\mathrm{p}^{\hat{\bm{A}}}\right|$. Then, we have the following lemma.

\begin{lemma} [Theorem 1. of~\citep{garg2021mixture}]
\label{lemma_err}
 For $\min(n_\mathrm{p}, n_\mathrm{u}) \ge \frac{2\log(4/\delta)}{\mathbbm{Q}_\mathrm{p}(c^*)}$ and for every $\delta >0$, the mixture proportion estimator $\hat{\pi}_p$ satisfies with probability $1-\delta$:  
\begin{equation}
      \left| \hat{\pi}_\mathrm{p} - \pi_\mathrm{p} \right|  \le  \frac{c}{\mathbbm{Q}_\mathrm{p}(c^*)}\left( \sqrt{\frac{\log(4/\delta)}{n_\mathrm{u}}} + \sqrt{\frac{\log(4/\delta)}{n_\mathrm{p}}}\right)  
\end{equation}    
for some constant $c\ge0$.
\end{lemma}

Furthermore, according to the Lemma~\ref{lemma_err}, we have the theorem on estimation error, which characterizes the benefits.
\begin{theorem} [Estimation Error]
\label{est_error}
Let the class prior can be estimated by  $\hat{\pi}_\mathrm{p} = \min_{\bm{c} \in [0,1]}\mathbbm{Q}^{\bm{A}}_\mathrm{u}(\bm{c})/{\mathbbm{Q}^{\bm{A}}_\mathrm{p}(\bm{c})}$ and the graph structure can be optimized through $\hat{\bm{A}} = \arg\min_{\bm{A}} \mathcal{L}_{\textit{LPA}}(\bm{A})^{(K)}$. Then, the upper bound of the estimation error with optimized graph structure $\hat{\bm{A}}$ is lower than ordinary graph structure $\bm{A}$.
\end{theorem}
Please refer to Appendix~\ref{proff_est_error} for the complete proof. According to Theorem~\ref{est_error}, we can find that our heterophily minimization makes the upper bound of class prior estimation tighten, which proves the benefits of reducing heterophilic structures to the vanilla graph PU learning.

\begin{table*} [t!]

    \centering
    \caption{
    Mean F1 score $\pm$ stdev over different datasets. The best model is highlighted in LightGreen. The heterophily (Hete.) ratios of all datasets are collected from~\cite{zhu2020beyond}. The ``*'' methods mean using the knowledge of class prior.
    }
    \vspace{-5pt}
    \label{tab:class_evl}
    \resizebox{\linewidth}{!}{
    \begin{tabular}{lcccc| cccccc} 
    \toprule
     &  \texttt{\bf Cora}           &   \texttt{\bf Pubmed}           &   \texttt{\bf Citeseer}   &   \texttt{\bf Wiki-CS}             &   \texttt{\bf Cornell} & \texttt{\bf Chameleon}        &  \texttt{\bf Squirrel}   &   \texttt{\bf Actor}           &   \texttt{\bf Wisconsin}            &   \texttt{\bf Texas}  \\
     
       \textbf{Hete.\ ratio} $h$ & \textbf{0.19} & \textbf{0.20} & \textbf{0.26} & \textbf{0.35}  & \textbf{0.70} & \textbf{0.77} & \textbf{0.78} & \textbf{0.78} & \textbf{0.79} & \textbf{0.89} \\
		\textbf{\#Nodes $|\mathcal{V}|$}  & 2,708 & 19,717 & 3,327& 11,701 & 183& 2,277& 5,201& 7,600& 251& 183\\
		\textbf{\#Edges $|\mathcal{E}|$} &  13,264 & 
108,365 & 12,431 &431,206  & 298 & 36,101 & 217,073 & 30,019 & 515 & 325  \\
    \midrule
    \midrule

	   {GCN} & $25.2{\scriptstyle\pm5.2}$ & $19.7{\scriptstyle\pm3.1}$ & $25.7{\scriptstyle\pm3.1}$ & $35.4{\scriptstyle\pm4.1}$  & $1.4{\scriptstyle\pm2.9}$ & $0.7{\scriptstyle\pm0.6}$ & $0.8{\scriptstyle\pm0.7}$ &  
    $0.6{\scriptstyle\pm0.3}$ & 
    $0.9{\scriptstyle\pm1.8}$ & 
    $1.3{\scriptstyle\pm3.0}$\\
	
	   {MLP} & $16.6{\scriptstyle\pm8.1}$ & $8.0{\scriptstyle\pm9.0}$ & $4.2{\scriptstyle\pm8.4}$ & $5.4{\scriptstyle\pm8.4}$  & $11.3{\scriptstyle\pm9.1}$ & $18.8{\scriptstyle\pm6.9}$ & $14.8{\scriptstyle\pm6.4}$ &  
    $17.6{\scriptstyle\pm8.1}$ & 
    $15.9{\scriptstyle\pm8.9}$ & 
    $11.3{\scriptstyle\pm9.2}$\\
        \midrule
    {GCN+TED} &  ${80.1\scriptstyle\pm0.8}$ & ${75.4\scriptstyle\pm0.4}$ & ${70.0\scriptstyle\pm1.4}$ &  ${80.6\scriptstyle\pm0.5}$  & ${13.9\scriptstyle\pm5.5}$ & 
    ${17.5\scriptstyle\pm3.5}$ & ${22.8\scriptstyle\pm1.3}$ & 
    ${27.6\scriptstyle\pm1.0}$ &  ${19.6\scriptstyle\pm6.1}$ & ${11.9\scriptstyle\pm5.0}$ \\
    {MLP+TED} &  
    ${24.4\scriptstyle\pm9.2}$ & 
    ${15.6\scriptstyle\pm19.1}$ & 
    ${10.6\scriptstyle\pm9.5}$ &  
    ${16.2\scriptstyle\pm7.5}$ & 
    ${29.2\scriptstyle\pm9.5}$ & 
    ${26.5\scriptstyle\pm4.1}$ & 
    ${31.4\scriptstyle\pm2.6}$ & 
    ${43.3\scriptstyle\pm6.7}$ & 
    ${40.4\scriptstyle\pm9.1}$ & 
    ${41.0\scriptstyle\pm7.9}$  \\
    {GCN+NNPU*} &  
    ${76.7\scriptstyle\pm0.9}$ & 
    ${76.5\scriptstyle\pm2.3}$ & 
    ${66.2\scriptstyle\pm1.1}$ &  
    ${71.1\scriptstyle\pm0.6}$ & 
    ${11.7\scriptstyle\pm6.3}$ & 
    ${28.2\scriptstyle\pm9.2}$ & 
    ${15.5\scriptstyle\pm12.6}$ & 
    ${25.2\scriptstyle\pm21.7}$ & 
    ${15.5\scriptstyle\pm8.6}$ & 
    ${22.8\scriptstyle\pm19.7}$  \\
    {MLP+NNPU*} &  
    ${25.2\scriptstyle\pm5.6}$ & 
    ${34.2\scriptstyle\pm2.3}$ & 
    ${12.9\scriptstyle\pm9.8}$ & 
    ${24.4\scriptstyle\pm5.1}$ & 
    ${32.9\scriptstyle\pm4.9}$ & 
    ${21.0\scriptstyle\pm3.7}$ & 
    ${32.1\scriptstyle\pm5.6}$ & 
    ${43.2\scriptstyle\pm2.6}$ & 
    ${35.6\scriptstyle\pm5.2}$ & 
    ${40.1\scriptstyle\pm6.2}$  \\
    
    \midrule
    
    {LSDAN*} &  
    ${63.5\scriptstyle\pm4.1}$ & ${69.6\scriptstyle\pm0.4}$ & ${47.0\scriptstyle\pm±19.}$ & ${80.4\scriptstyle\pm0.9}$ & 
    ${26.8\scriptstyle\pm4.3}$ & 
    ${30.2\scriptstyle\pm8.5}$ & 
    ${25.5\scriptstyle\pm9.3}$ & 
    ${38.4\scriptstyle\pm9.2}$ & 
    ${25.5\scriptstyle\pm8.6}$ & 
    ${41.8\scriptstyle\pm9.7}$  \\
    
    {GRAB} &  $80.4{\scriptstyle\pm0.2}$ &  ${71.6\scriptstyle\pm0.3}$
    & 
    $69.7{\scriptstyle\pm0.4}$ &  $79.4{\scriptstyle\pm1.0}$ &  $30.6{\scriptstyle\pm9.0}$ & ${15.5\scriptstyle\pm2.2}$ & ${29.0\scriptstyle\pm8.4}$  & ${25.2\scriptstyle\pm0.7}$ &  ${39.5\scriptstyle\pm9.6}$ &  ${40.2\scriptstyle\pm8.2}$ \\

    {PU-GNN*} &  
    ${79.8\scriptstyle\pm0.3}$ & 
    ${73.0\scriptstyle\pm0.4}$ & 
    ${69.5\scriptstyle\pm0.4}$ &  
    ${80.3\scriptstyle\pm1.8}$ & 
    ${28.3\scriptstyle\pm3.9}$ & 
    ${31.5\scriptstyle\pm6.3}$ & 
    ${29.5\scriptstyle\pm4.8}$ & 
    ${32.6\scriptstyle\pm8.4}$ & 
    ${27.8\scriptstyle\pm8.2}$ & 
    ${42.3\scriptstyle\pm9.4}$  \\
           \midrule
         \rowcolor{LightGreen}
        GPL & $\bf{81.9{\scriptstyle\pm0.5}}$ 
        &$\bf{79.1{\scriptstyle\pm0.4}}$
        &$\bf{74.1{\scriptstyle\pm0.9}}$
        &$\bf{82.3{\scriptstyle\pm0.8}}$
        &$\bf{37.9{\scriptstyle\pm1.3}}$
        &$\bf{36.2{\scriptstyle\pm1.4}}$
        &$\bf{37.7{\scriptstyle\pm5.8}}$
        &$\bf{48.0{\scriptstyle\pm4.2}}$
        &$\bf{45.2{\scriptstyle\pm2.7}}$
        &$\bf{46.3{\scriptstyle\pm3.2}}$
\\
          \rowcolor{LightGreen}
         & \textbf{\color{Green}{$\uparrow$ $2.2\%$}}
        &\textbf{\color{Green}{$\uparrow$ $3.4\%$}}
        &\textbf{\color{Green}{$\uparrow$ $5.8\%$}}
        &\textbf{\color{Green}{$\uparrow$ $2.1\%$}}
        &\textbf{\color{Green}{$\uparrow$ $15.2\%$}}
        & \textbf{\color{Green}{$\uparrow$ $14.5\%$}}
        & \textbf{\color{Green}{$\uparrow$ $17.4\%$}}
        &\textbf{\color{Green}{$\uparrow$ $10.9\%$}}
        &\textbf{\color{Green}{$\uparrow$ $12.5\%$}}
        & \textbf{\color{Green}{$\uparrow$ $10.2\%$}}\\

	   \bottomrule
    \end{tabular}
    }
\end{table*}
\vspace{-5pt}

\subsection{Why heterophilic structure degenerates PU Classification?}

In this part, we perform a theoretical analysis of the heterophilic structure, elucidating how it leads P and N nodes to converge closer to each other within the embedding space. Let ${\bm h}^{(l)}$ be the embedding of node ${\bm x}^{(l)}$ after trained graph classifier. Through the following theorem, we illustrate that the aggregation step in graph classifier training, influenced by the heterophilic structure, reduces the distance in the embedding space between P and N nodes.

\begin{theorem}
\label{thm:decrease}
    Let $D_{\mathrm{P}\mathrm{N}}({\bm{x}}) = \frac{1}{2} \sum_{\bm{x}_i \in \mathrm{P}, \bm{x}_j \in \mathrm{N}} \tilde{\bm{A}}_{ij} \| {\bm{x}}_i - {\bm{x}}_j \|_2^2$ be a distance metric between P and N node embeddings $\bm{x}$ and $\tilde{\bm{A}}_{ij}(\bm{x}_i \in \mathrm{P}, \bm{x}_j \in \mathrm{N})$ is the normalized weight of heterophilic edges. Then we have
    \begin{equation*}
        D_{\mathrm{P}\mathrm{N}}({\bm h}^{(l)}) \leq D_{\mathrm{P}\mathrm{N}}({\bm{x}}^{(l)}).
    \end{equation*}
\end{theorem}

Proof of Theorem \ref{thm:decrease} is in Appendix \ref{app:d}. Theorem \ref{thm:decrease} indicates that following one aggregation step within the heterophilic
structure, there is a reduction in the overall distance between P and N nodes. Thus, the heterophilic edges between P and N nodes hinder the distinguishability of these nodes, thereby making the task of identifying P and N nodes from U nodes more challenging during PU classification. Fortunately, the GPL method effectively minimizes heterophilic edge weights by utilizing LPL, which ensures that P and N nodes maintain a more discernible distinction than that they would be under the original structure. 


\section{Experiment}

\subsection{Experiment Setting} 
\begin{table*} [t!]
    \centering
    \caption{Absolute estimation error with the true class prior in the first row over different datasets. Results were reported by meaning absolute error over five experiments. The best model with the smallest estimation error is highlighted.}
    \vspace{-5pt}
    \label{tab:cpe_es}
    \resizebox{\linewidth}{!}{
    \begin{tabular}{lcccc| cccccc} 
    \toprule
     &  \texttt{\bf Cora}           &   \texttt{\bf Pubmed}           &   \texttt{\bf Citeseer}   &   \texttt{\bf Wiki-CS}          &   \texttt{\bf Cornell} & \texttt{\bf Chameleon}        &  \texttt{\bf Squirrel}   &   \texttt{\bf Actor}           &   \texttt{\bf Wisconsin}            &   \texttt{\bf Texas}  \\
     
     \textbf{Hete.\ ratio} $h$ & \textbf{0.19} & \textbf{0.20} & \textbf{0.26} & \textbf{0.35}  & \textbf{0.70} & \textbf{0.77} & \textbf{0.78} & \textbf{0.78} & \textbf{0.79} & \textbf{0.89} \\
     
		\textbf{\#True $\pi_{\mathrm{p}}$} &  0.1779 & 
0.2496 & 0.1179 & 0.1293 & 0.2887 & 0.1294 & 0.1113 & 0.1485 & 0.3073 & 0.3835  \\
    \midrule
    \midrule

        {KM} &  
    ${0.09 \scriptstyle\pm 0.02}$ & 
    ${0.08 \scriptstyle\pm 0.04}$ & 
    ${0.10\scriptstyle\pm 0.05}$ &  
    ${0.06 \scriptstyle\pm 0.03}$ & 
    ${0.21 \scriptstyle\pm 0.04}$ & 
    ${0.18 \scriptstyle\pm 0.05}$ & 
    ${0.39\scriptstyle\pm  0.10}$ & 
    ${0.16 \scriptstyle\pm  0.06}$ & 
    ${0.19 \scriptstyle\pm 0.08}$ & 
    ${0.31\scriptstyle\pm 0.05}$  \\

            {DEDPUL} &  
    ${0.05 \scriptstyle\pm 0.03}$ & 
    ${0.05 \scriptstyle\pm 0.05}$ & 
    ${0.07\scriptstyle\pm 0.04}$ &  
    ${0.05 \scriptstyle\pm 0.02}$ & 
    ${0.19 \scriptstyle\pm 0.05}$ & 
    ${0.16 \scriptstyle\pm 0.04}$ & 
    ${0.32\scriptstyle\pm  0.09}$ & 
    ${0.14 \scriptstyle\pm  0.07}$ & 
    ${0.17 \scriptstyle\pm 0.06}$ & 
    ${0.25\scriptstyle\pm 0.07}$  \\

    {MPE} &  
    ${0.04 \scriptstyle\pm 0.03}$ & 
    ${0.04 \scriptstyle\pm 0.02}$ & 
    ${0.08\scriptstyle\pm 0.04}$ &  
    ${0.02 \scriptstyle\pm 0.01}$ & 
    ${0.13 \scriptstyle\pm 0.06}$ & 
    ${0.10 \scriptstyle\pm 0.03}$ & 
    ${0.36\scriptstyle\pm  0.17}$ & 
    ${0.08 \scriptstyle\pm  0.02}$ & 
    ${0.14 \scriptstyle\pm 0.06}$ & 
    ${0.26\scriptstyle\pm 0.03}$  \\

    {ReMPE} &  
    ${0.03 \scriptstyle\pm 0.03}$ & 
    ${0.04 \scriptstyle\pm 0.01}$ & 
    ${0.06\scriptstyle\pm 0.03}$ &  
    ${0.02 \scriptstyle\pm 0.02}$ & 
    ${0.14 \scriptstyle\pm 0.07}$ & 
    ${0.12 \scriptstyle\pm 0.05}$ & 
    ${0.34\scriptstyle\pm  0.12}$ & 
    ${0.11 \scriptstyle\pm  0.04}$ & 
    ${0.12 \scriptstyle\pm 0.05}$ & 
    ${0.25\scriptstyle\pm 0.03}$  \\

    {BBE} & ${0.03\scriptstyle\pm0.01}$ & 
    ${0.05\scriptstyle\pm0.01}$ & 
    ${0.04\scriptstyle\pm0.02}$ &  
    ${0.03\scriptstyle\pm0.01}$ & 
    ${0.27\scriptstyle\pm0.01}$ & 
    ${0.12\scriptstyle\pm0.01}$ & 
    ${0.10\scriptstyle\pm0.08}$ & 
    ${0.11\scriptstyle\pm0.01}$ & 
    ${0.23\scriptstyle\pm0.05}$ & 
    ${0.27\scriptstyle\pm0.02}$  \\

    {TED} &  ${0.02\scriptstyle\pm0.01}$ & ${0.04\scriptstyle\pm0.00}$ & ${0.02\scriptstyle\pm0.01}$ &  ${0.02\scriptstyle\pm0.01}$ &  ${0.18\scriptstyle\pm0.04}$ & 
    ${0.25\scriptstyle\pm0.03}$ & ${0.13\scriptstyle\pm0.04}$ & 
    ${0.25\scriptstyle\pm0.10}$ &  ${0.16\scriptstyle\pm0.02}$ & ${0.26\scriptstyle\pm0.03}$ \\

    {GRAB} &  
    ${0.07\scriptstyle\pm0.13}$ & 
    ${0.08\scriptstyle\pm0.09}$ & 
    ${0.06\scriptstyle\pm0.08}$ &  
    ${0.09\scriptstyle\pm0.12}$ & 
    ${0.35\scriptstyle\pm 0.26}$ & 
    ${0.18 \scriptstyle\pm 0.07}$ & 
    ${0.36\scriptstyle\pm 0.19}$ & 
    ${0.30 \scriptstyle\pm 0.07}$ & 
    ${0.37\scriptstyle\pm 0.18}$ & 
    ${0.33\scriptstyle\pm 0.20}$  
      \\
           \midrule

         \rowcolor{LightGreen}
        GPL & $\bf{0.01{\scriptstyle\pm0.01}}$&$\bf{0.008{\scriptstyle\pm0.01}}$&$\bf{0.01{\scriptstyle\pm0.01}}$&$\bf{0.01{\scriptstyle\pm0.00}}$&$\bf{0.04{\scriptstyle\pm0.03}}$&$\bf{0.07±  {\scriptstyle\pm0.01}}$&$\bf{0.02  {\scriptstyle\pm0.01}}$&$\bf{0.02{\scriptstyle\pm0.01}}$&$\bf{0.03{\scriptstyle\pm0.02}}$&$\bf{0.08{\scriptstyle\pm0.03}}$ \\
            \rowcolor{LightGreen}
         & \textbf{\color{Green}{$\downarrow$ $0.01$}}
        &\textbf{\color{Green}{$\downarrow$ $0.03$}}
        &\textbf{\color{Green}{$\downarrow$ $0.01$}}
        &\textbf{\color{Green}{$\downarrow$ $0.01$}}
        &\textbf{\color{Green}{$\downarrow$ $0.09$}}
        & \textbf{\color{Green}{$\downarrow$ $0.03$}}
        & \textbf{\color{Green}{$\downarrow$ $0.08$}}
        &\textbf{\color{Green}{$\downarrow$ $0.06$}}
        &\textbf{\color{Green}{$\downarrow$  $0.09$}}
        & \textbf{\color{Green}{$\downarrow$ $0.17$}}\\
        
	   \bottomrule
    \end{tabular}
    }
    \vspace{-5pt}
\end{table*}
\paragraph{Datasets}
We evaluate the performance of our method on various real-world datasets, each characterized by an edge homophily ratio $h$ ranging from strong homophily to strong heterophily, as defined in ~\citet{zhu2020beyond}. We have summarized the dataset details in Table~\ref{tab:class_evl}. To transform these datasets
into binary classification tasks, we follow the previous approach ~\citep{yoo2021accurate,yang2023positive}, where we treat the label with the largest number of nodes as positive and the rest as negative. The resulting numbers of positive and negative nodes are reported in the Appendix Table~\ref{table::dataset}.

\vspace{-5pt}
\paragraph{Dataset Processing}
To ensure a fair comparison with related work~\cite{yoo2021accurate,wu2021learning,yang2023positive}, we randomly split each positive and negative dataset into positive and unlabeled sets. For each dataset, we use $r_{\mathrm{p}} = 50\%$ of all positive nodes as the observed positive nodes $\mathcal{P}$ and treat the rest as unobserved positive nodes $\mathcal{P}_{u}$. All negative nodes are treated as unobserved and denoted by $\mathcal{N}_{u}$, as we assume graph PU learning. Then, our objective is to predict the labels of each node in $\mathcal{U} = \mathcal{P}_{u}\cup \mathcal{N}_{u}$ as test data. This is achieved by training a classifier using the labeled nodes $\mathcal{P}$ and the unlabeled nodes $\mathcal{U}$; All nodes are accessible during training, but the labels of only $\mathcal{P}$ are observable as training data. Additionally, the true prior ${\pi}_{p}$, which is defined as ${\pi}_{p} = \mathcal{P}_{u}/(\mathcal{P}_{u}+\mathcal{N}_{u})$ in our experiments.

\vspace{-5pt}
\paragraph{Experimental Setup}
In the GPL, the backbone model is a graph convolutional network (GCN), with the number of layers set to $2$ and the size of hidden layers set to $16$. We train each model using Adam optimizer with a learning rate of $0.01$. Owing to the imbalance between positive and negative nodes in our dataset, we adopt F1 score as our measure. For each experiment, we run five trials with different random seeds and  compute the average and standard deviation.

\subsection{Main Results}

\begin{figure*} [t!]

\centering
\includegraphics[width=1\linewidth]{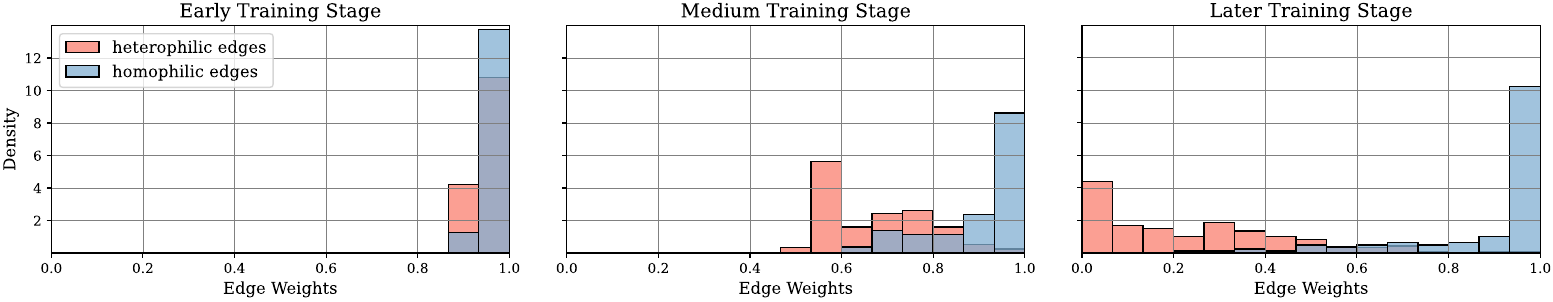}
\vspace{-20pt}

\caption{Distribution of edge weights for heterophilic and homophilic edges on Cornell dataset as training process in GPL. 
}
\label{figure::distribution_hete}
\vspace{-5pt}
\end{figure*}

\paragraph{Classification Evaluation}
In this experiment, we benchmarked our method against several established approaches. This includes (1) Native GCN and MLP, which represent standard baselines for a range of graph structures, from homophily to heterophily; (2) GCN+TED, MLP+TED, GCN+NNPU, and MLP+NNPU: These are two popular methods adapted for PU learning in i.i.d data, \ie TED~\citep{garg2021mixture} and NNPU~\cite{kiryo2017positive} implemented with GCN and MLP as the backbone models. (3) LSDAN~\citep{ma2017pu}, GRAB~\citep{yoo2021accurate}, and PU-GNN~\citep{yang2023positive}: These are three current methods tailored for graph PU learning. Further details about these baseline methods can be found in Appendix~\ref{baseline}. 

The results presented in Table~\ref{tab:class_evl} show that our GPL method consistently outperforms others, achieving the highest classification performance across a variety of datasets, from homophilic to heterophilic datasets. Remarkably, our method shows particularly substantial improvements in datasets characterized by a high degree of heterophilic structure. Popular methods that excel in PU learning with i.i.d. data, such as TED and NNPU, do not demonstrate the same level of effectiveness on graph data, especially in the presence of pronounced heterophilic structures. This discrepancy underscores the unique challenges posed by heterophilic structures and proves the necessity and novelty of our GPL approach. Moreover, existing graph-based PU learning methods also show a marked decline in performance under heterophilic conditions, further confirming this point.

\vspace{-10pt}
\paragraph{Class Prior Estimation}
In this section, we discuss the results of CPE. We compare our method with KM~\citep{ramaswamy2016mixture}, DEDPUL~\citep{ivanov2020dedpul}, MPE~\citep{scott2015rate}, REMPE~\citep{yao2020rethinking}, BBE, TED~\citep{garg2021mixture} and GRAB~\citep{yoo2021accurate}. The summarized results are shown in Table~\ref{tab:cpe_es}. Overall, using the same backbone model, our method consistently outperforms other baseline models. Notably, as the heterophily ratio increases, we observe a corresponding decrease in the estimation error. This trend clearly demonstrates that our method effectively reduces estimation errors by mitigating the impact of the heterophilic structure.

\begin{figure}[t!]
\centering
\includegraphics[width=1\linewidth]{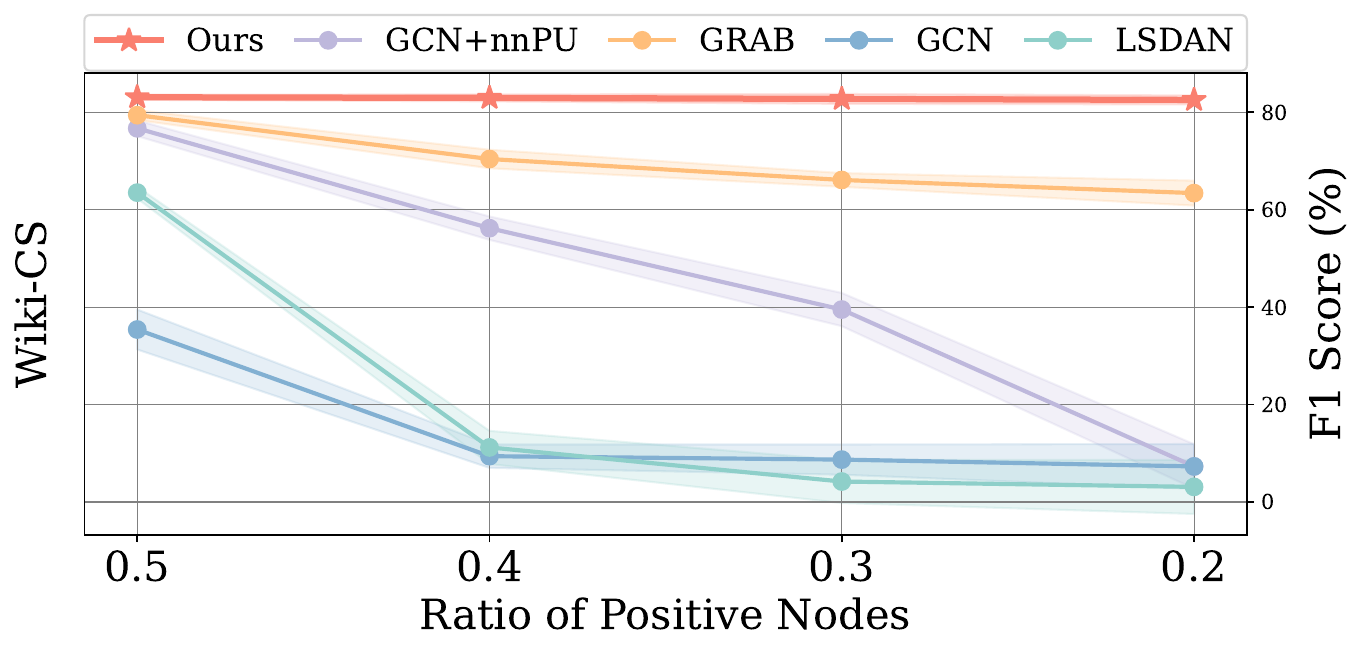}
\vspace{-18pt}
\caption{ The F1 scores of GPL and the baseline approaches for PU learning. We change the ratio $r_{\mathrm{p}}$ of observed positive nodes among all positive ones.}
\label{figure::diff_ratio}
\end{figure}

\subsection{Ablation Study}

\paragraph{Different Ratios of Labeled Nodes in PU learning.}
We summarize the performance of GPL and the baselines with different ratios of observed positive nodes to unlabeled ones in Figure~\ref{figure::diff_ratio}. We gradually decrease the ratio $r_{\mathrm{p}}$ from $0.5$ to $0.2$. The problem becomes more difficult with smaller $r_{\mathrm{p}}$. Despite this challenge, our method consistently maintains high F1 scores across all values of $r_{\mathrm{p}}$, unlike the baselines whose performance deteriorates. This indicates the robustness of GPL even in imbalanced PU learning scenarios, where the count of observed nodes is significantly lower than the total number of nodes.

\vspace{-5pt}
\paragraph{The distribution of Learned weights}
In Figure~\ref{figure::distribution_hete}, we showcase the distribution of weights for both heterophilic and homophilic edges. During training with GPL, heterophilic edges tend to receive lower weights, whereas the weights of homophilic edges remain larger. This empirical observation demonstrates the substantial effectiveness of our GPL in mitigating the influence of heterophilic structures, thereby contributing to graph PU learning.
\vspace{-5pt}
\paragraph{Contribution of Each Component}
We conducted a comprehensive analysis to evaluate the individual contributions of various components in our GPL method. This analysis in Table~\ref{tab:ablation} includes (1) GPL without the LPL, focusing solely on reducing the heterophilic structure; (2) GPL without bilevel optimization,  employing a two-stage approach of first optimizing the graph structure and then learning from PU nodes; (3) GPL without selected nodes, omitting the use of extracted nodes to enhance the convergence of the LPL. The results elucidate the distinct contributions of each component and underscore the importance of the LPL, especially in scenarios with a high heterophily ratio.

\begin{table} [t!]
    \centering
    \caption{Ablation study for GPL.}
    \vspace{-5pt}
    \label{tab:ablation}
    \resizebox{\linewidth}{!}{
    \begin{tabular}{c|cccc} 
    \toprule
    \texttt{\bf Variant} &
         \texttt{\bf Citeseer}   &    \texttt{\bf Wiki-CS} 
        &   \texttt{\bf Cornell}   &    \texttt{\bf Chameleon}    \\
\textbf{Hete.\ ratio} $h$   & \textbf{0.26} & \textbf{0.35} & \textbf{0.70} & \textbf{0.77} \\
    \midrule
    \midrule

         \rowcolor{LightGreen}
        GPL 
           &$\bf{74.1{\scriptstyle\pm0.9}}$
        &$\bf{82.3{\scriptstyle\pm0.8}}$
        &$\bf{37.9{\scriptstyle\pm1.3}}$
        &$\bf{36.2{\scriptstyle\pm1.4}}$
        
\\
     - w/o LPL &     ${70.0\downarrow 5.5\%}$ & 
    ${79.6\downarrow 3.3\%}$ & 
    ${13.9\downarrow 63.3\%}$ &  
    ${17.5\downarrow 51.6\%}$  \\

    - w/o Bilevel &          ${71.5\downarrow 3.5\%}$ & 
    ${80.1\downarrow 2.7\%}$ & 
    ${32.4\downarrow 14.5\%}$ & 
    ${28.9\downarrow 20.2\%}$    \\

    - w/o Selected &     ${71.1\downarrow 5.3\%}$ & 
    ${80.2\downarrow 3.7\%}$ & 
    ${34.6\downarrow 8.7\%}$ &  
    ${32.4\downarrow 10.5\%}$ \\
        
	   \bottomrule
    \end{tabular}
    }
\end{table}

\section{Related Work}

\paragraph{Positive-Unlabeled Learning}
Positive-Unlabeled (PU) learning is broadly categorized into two subtasks: (i) Class-Prior
Estimation (CPE) — determining the fraction of positive
examples in the unlabeled data; and (ii) PU classification—given such an estimate, learning the desired positive-versus-negative classifier. Research on CPE and PU classification date to ~\citep{de1999positive,denis1998pac,letouzey2000learning}. Without specific assumptions for CPE, class prior remains unidentifiable. To ensure identifiability, \citet{du2014semi} and \citet{elkan2008learning} posited that positive and negative examples must have disjoint support. Following this, the concept of the irreducibility assumption was introduced by \citet{blanchard2010semi}, forming the basis for nearly all subsequent CPE algorithms \citep{blanchard2010semi, ivanov2020dedpul, jain2016nonparametric}. Recently, \citet{yao2020rethinking} and \citet{zhu2023mixture} have begun exploring approaches to CPE that transcend the irreducibility assumption. However, these assumptions are based on i.i.d data, making them ill-suited for application to graph data. Our work is the first effort to adapt these assumptions for graph-based CPE. 

Given the estimated class prior, the PU classification methods can be divided into two categories based on how unlabeled data is handled~\citep{xinrui2023beyond}. The first category focuses on the sample-selection task to form a reliable negative set and further yield the semi-supervised learning framework, where the quantity of selected N data depends on the class prior~\citep{liu2002partially,li2003learning}; the second regards unlabeled data as weighted positive and negative data simultaneously, as determined by the class-prior~\citep{niu2016theoretical,kiryo2017positive}. Presently, methods specific to graph-based PU classification all fall into the latter category~\citep{ma2017pu,yoo2021accurate,yang2023positive}. However, these methods typically rely on the homophily characteristics of graphs to modulate the training loss for positive and unlabeled nodes, restricting their effectiveness in heterophilic structures, which are prevalent in graph data~\citep{zhu2022does,li2022finding}.

\paragraph{Graph Learning with Heterophily}
Heterophilic structures, prevalent in many graph data scenarios, have garnered considerable attention lately, leading to the emergence of various models designed to tackle this challenge.  For instance, Geom-GCN~\citep{pei2020geom} employs a bilevel aggregation process in the embedding space to manage heterophily. CPGNN~\citep{zhu2021graph} addresses heterophilic signals by modelling label correlations using a compatibility matrix. Other methods~\citep{luan2022revisiting,chien2020adaptive,bo2021beyond} using the high-frequency graph signals in supervised embedding space to address heterophily. However, a common limitation of these methods is their dependency on complete label information to address heterophily. This dependency is not available in our PU learning setting, where negative labels are absent.

\section{Conclusion}

In this paper, we address the challenge of employing Positive and Unlabeled (PU) learning on graph-structured data, primarily due to the presence of heterophilic edges. We introduce the \textit{\underline{G}raph \underline{P}U Learning with \underline{L}abel \underline{P}ropagation \underline{L}oss} (GPL) method, specifically engineered to mitigate the adverse effects of heterophilic structures. This mitigation hinges on the newly proposed \textit{Label Propagation Loss} (LPL), which effectively lowers the weights of heterophilic edges. Assisted by LPL, GPL refines the graph structure and develops a well-performing binary classifier for positive-unlabeled nodes. Our approach encapsulates this process in a bilevel optimization framework, where heterophilic reduction occurs in the inner loop, and classifier learning takes place in the outer loop. We have validated the efficacy of this framework through rigorous theoretical analysis and comprehensive experiments. Looking forward, we aim to extend our exploration to other forms of imperfect graph data, such as imbalanced or out-of-distribution graphs, further demonstrating the versatility of our approach.

\section*{Acknowledgements}

TLL is partially supported by the following Australian Research Council projects: FT220100318, DP220102121, LP220100527, LP220200949, and IC190100031. JCY is supported by the National Key R\&D Program of China (No. 2022ZD0160703), 111 plan (No. BP0719010) and National Natural Science Foundation of China (No. 62306178). BH is supported by the NSFC General Program No. 62376235 and Guangdong Basic and Applied Basic Research Foundation Nos. 2022A1515011652 and 2024A1515012399. The authors would give special thanks to Suqin Yuan and Muyang Li for helpful discussions and comments. The authors thank the reviewers and the meta-reviewer for their helpful and constructive comments on this work.

\section*{Impact Statement}

In the era of big data, the prevalence of incomplete (positive-unlabeled) labels presents significant reliability challenges for traditional supervised learning algorithms. This issue is particularly acute in graph data, where it impedes the effective training of graph models. Learning models accurately from positive and unlabeled nodes is a crucial issue, drawing increasing attention in both research and industry circles due to its significant impact on practical applications involving graph data.

In this study, we propose the GPL method as a novel approach to learning from positive and unlabeled nodes. This method is specifically designed to counter the negative effects of heterophilic edges. The effectiveness of GPL is substantiated by the extensive evidence presented in our paper. The findings of this research contribute to a deeper understanding of how to handle incomplete labels in graph data. They mark a significant step forward in enhancing the robustness and accuracy of graph models, paving the way for more reliable and precise graph-based learning.

\bibliography{main}

\begin{thebibliography}{59}
\providecommand{\natexlab}[1]{#1}
\providecommand{\url}[1]{\texttt{#1}}
\expandafter\ifx\csname urlstyle\endcsname\relax
  \providecommand{\doi}[1]{doi: #1}\else
  \providecommand{\doi}{doi: \begingroup \urlstyle{rm}\Url}\fi

\bibitem[Bekker \& Davis(2020)Bekker and Davis]{bekker2020learning}
Bekker, J. and Davis, J.
\newblock Learning from positive and unlabeled data: A survey.
\newblock \emph{Machine Learning}, 109:\penalty0 719--760, 2020.

\bibitem[Bianchi et~al.(2021)Bianchi, Grattarola, Livi, and Alippi]{bianchi2021graph}
Bianchi, F.~M., Grattarola, D., Livi, L., and Alippi, C.
\newblock Graph neural networks with convolutional arma filters.
\newblock \emph{IEEE transactions on pattern analysis and machine intelligence}, 44\penalty0 (7):\penalty0 3496--3507, 2021.

\bibitem[Blanchard et~al.(2010)Blanchard, Lee, and Scott]{blanchard2010semi}
Blanchard, G., Lee, G., and Scott, C.
\newblock Semi-supervised novelty detection.
\newblock \emph{The Journal of Machine Learning Research}, 11:\penalty0 2973--3009, 2010.

\bibitem[Bo et~al.(2021)Bo, Wang, Shi, and Shen]{bo2021beyond}
Bo, D., Wang, X., Shi, C., and Shen, H.
\newblock Beyond low-frequency information in graph convolutional networks.
\newblock In \emph{AAAI}, volume~35, pp.\  3950--3957, 2021.

\bibitem[Chen et~al.(2022)Chen, Li, Liu, and Hong]{chen2022characterizing}
Chen, Z., Li, P., Liu, H., and Hong, P.
\newblock Characterizing the influence of graph elements.
\newblock \emph{arXiv preprint arXiv:2210.07441}, 2022.

\bibitem[Chien et~al.(2020)Chien, Peng, Li, and Milenkovic]{chien2020adaptive}
Chien, E., Peng, J., Li, P., and Milenkovic, O.
\newblock Adaptive universal generalized pagerank graph neural network.
\newblock \emph{arXiv preprint arXiv:2006.07988}, 2020.

\bibitem[De~Comit{\'e} et~al.(1999)De~Comit{\'e}, Denis, Gilleron, and Letouzey]{de1999positive}
De~Comit{\'e}, F., Denis, F., Gilleron, R., and Letouzey, F.
\newblock Positive and unlabeled examples help learning.
\newblock In \emph{ALT}, pp.\  219--230. Springer, 1999.

\bibitem[Denis(1998)]{denis1998pac}
Denis, F.
\newblock Pac learning from positive statistical queries.
\newblock In \emph{ALT}, pp.\  112--126. Springer, 1998.

\bibitem[Denis et~al.(2005)Denis, Gilleron, and Letouzey]{denis2005learning}
Denis, F., Gilleron, R., and Letouzey, F.
\newblock Learning from positive and unlabeled examples.
\newblock \emph{Theoretical Computer Science}, 348\penalty0 (1):\penalty0 70--83, 2005.

\bibitem[Du~Plessis \& Sugiyama(2014)Du~Plessis and Sugiyama]{du2014semi}
Du~Plessis, M.~C. and Sugiyama, M.
\newblock Semi-supervised learning of class balance under class-prior change by distribution matching.
\newblock \emph{Neural Networks}, 50:\penalty0 110--119, 2014.

\bibitem[Elkan \& Noto(2008)Elkan and Noto]{elkan2008learning}
Elkan, C. and Noto, K.
\newblock Learning classifiers from only positive and unlabeled data.
\newblock In \emph{KDD}, pp.\  213--220, 2008.

\bibitem[Garg et~al.(2021)Garg, Wu, Smola, Balakrishnan, and Lipton]{garg2021mixture}
Garg, S., Wu, Y., Smola, A.~J., Balakrishnan, S., and Lipton, Z.
\newblock Mixture proportion estimation and pu learning: a modern approach.
\newblock \emph{NeurIPS}, 34:\penalty0 8532--8544, 2021.

\bibitem[Gasteiger et~al.(2018)Gasteiger, Bojchevski, and G{\"u}nnemann]{gasteiger2018predict}
Gasteiger, J., Bojchevski, A., and G{\"u}nnemann, S.
\newblock Predict then propagate: Graph neural networks meet personalized pagerank.
\newblock \emph{arXiv preprint arXiv:1810.05997}, 2018.

\bibitem[Gaudelet et~al.(2021)Gaudelet, Day, Jamasb, Soman, Regep, Liu, Hayter, Vickers, Roberts, Tang, et~al.]{gaudelet2021utilizing}
Gaudelet, T., Day, B., Jamasb, A.~R., Soman, J., Regep, C., Liu, G., Hayter, J.~B., Vickers, R., Roberts, C., Tang, J., et~al.
\newblock Utilizing graph machine learning within drug discovery and development.
\newblock \emph{Briefings in bioinformatics}, 22\penalty0 (6):\penalty0 bbab159, 2021.

\bibitem[Hu et~al.(2020)Hu, Fey, Zitnik, Dong, Ren, Liu, Catasta, and Leskovec]{hu2020open}
Hu, W., Fey, M., Zitnik, M., Dong, Y., Ren, H., Liu, B., Catasta, M., and Leskovec, J.
\newblock Open graph benchmark: Datasets for machine learning on graphs.
\newblock \emph{NeurIPS}, 33:\penalty0 22118--22133, 2020.

\bibitem[Huang et~al.(2023)Huang, Li, Huang, Chen, and Zhang]{huang2023revisiting}
Huang, J., Li, P., Huang, R., Chen, N., and Zhang, A.
\newblock Revisiting the role of heterophily in graph representation learning: An edge classification perspective.
\newblock \emph{ACM Transactions on Knowledge Discovery from Data}, 2023.

\bibitem[Ivanov(2020)]{ivanov2020dedpul}
Ivanov, D.
\newblock Dedpul: Difference-of-estimated-densities-based positive-unlabeled learning.
\newblock In \emph{ICMLA}, pp.\  782--790. IEEE, 2020.

\bibitem[Jain et~al.(2016)Jain, White, Trosset, and Radivojac]{jain2016nonparametric}
Jain, S., White, M., Trosset, M.~W., and Radivojac, P.
\newblock Nonparametric semi-supervised learning of class proportions.
\newblock \emph{arXiv preprint arXiv:1601.01944}, 2016.

\bibitem[Kipf \& Welling(2016)Kipf and Welling]{kipf2016semi}
Kipf, T.~N. and Welling, M.
\newblock Semi-supervised classification with graph convolutional networks.
\newblock \emph{arXiv preprint arXiv:1609.02907}, 2016.

\bibitem[Kiryo et~al.(2017)Kiryo, Niu, Du~Plessis, and Sugiyama]{kiryo2017positive}
Kiryo, R., Niu, G., Du~Plessis, M.~C., and Sugiyama, M.
\newblock Positive-unlabeled learning with non-negative risk estimator.
\newblock \emph{NeurIPS}, 30, 2017.

\bibitem[Koh \& Liang(2017)Koh and Liang]{koh2017understanding}
Koh, P.~W. and Liang, P.
\newblock Understanding black-box predictions via influence functions.
\newblock In \emph{ICML}, pp.\  1885--1894. PMLR, 2017.

\bibitem[Letouzey et~al.(2000)Letouzey, Denis, and Gilleron]{letouzey2000learning}
Letouzey, F., Denis, F., and Gilleron, R.
\newblock Learning from positive and unlabeled examples.
\newblock In \emph{ALT}, pp.\  71--85. Springer, 2000.

\bibitem[Li et~al.(2016)Li, Pan, Zhang, and Cai]{li2016classifying}
Li, M., Pan, S., Zhang, Y., and Cai, X.
\newblock Classifying networked text data with positive and unlabeled examples.
\newblock \emph{Pattern Recognition Letters}, 77:\penalty0 1--7, 2016.

\bibitem[Li \& Liu(2003)Li and Liu]{li2003learning}
Li, X. and Liu, B.
\newblock Learning to classify texts using positive and unlabeled data.
\newblock In \emph{IJCAI}, volume~3, pp.\  587--592. Citeseer, 2003.

\bibitem[Li et~al.(2022)Li, Zhu, Cheng, Shan, Luo, Li, and Qian]{li2022finding}
Li, X., Zhu, R., Cheng, Y., Shan, C., Luo, S., Li, D., and Qian, W.
\newblock Finding global homophily in graph neural networks when meeting heterophily.
\newblock In \emph{ICML}, pp.\  13242--13256. PMLR, 2022.

\bibitem[Liu et~al.(2002)Liu, Lee, Yu, and Li]{liu2002partially}
Liu, B., Lee, W.~S., Yu, P.~S., and Li, X.
\newblock Partially supervised classification of text documents.
\newblock In \emph{ICML}, volume~2, pp.\  387--394. Sydney, NSW, 2002.

\bibitem[Liu \& Tao(2015)Liu and Tao]{liu2015classification}
Liu, T. and Tao, D.
\newblock Classification with noisy labels by importance reweighting.
\newblock \emph{IEEE Transactions on pattern analysis and machine intelligence}, 38\penalty0 (3):\penalty0 447--461, 2015.

\bibitem[Luan et~al.(2022)Luan, Hua, Lu, Zhu, Zhao, Zhang, Chang, and Precup]{luan2022revisiting}
Luan, S., Hua, C., Lu, Q., Zhu, J., Zhao, M., Zhang, S., Chang, X.-W., and Precup, D.
\newblock Revisiting heterophily for graph neural networks.
\newblock \emph{NeurIPS}, 35:\penalty0 1362--1375, 2022.

\bibitem[Ma \& Zhang(2017)Ma and Zhang]{ma2017pu}
Ma, S. and Zhang, R.
\newblock Pu-lp: A novel approach for positive and unlabeled learning by label propagation.
\newblock In \emph{ICMEW}, pp.\  537--542. IEEE, 2017.

\bibitem[Ma et~al.(2021)Ma, Liu, Shah, and Tang]{ma2021homophily}
Ma, Y., Liu, X., Shah, N., and Tang, J.
\newblock Is homophily a necessity for graph neural networks?
\newblock \emph{arXiv preprint arXiv:2106.06134}, 2021.

\bibitem[Niu et~al.(2016)Niu, Du~Plessis, Sakai, Ma, and Sugiyama]{niu2016theoretical}
Niu, G., Du~Plessis, M.~C., Sakai, T., Ma, Y., and Sugiyama, M.
\newblock Theoretical comparisons of positive-unlabeled learning against positive-negative learning.
\newblock \emph{NeurIPS}, 29, 2016.

\bibitem[Panagopoulos et~al.(2021)Panagopoulos, Nikolentzos, and Vazirgiannis]{panagopoulos2021transfer}
Panagopoulos, G., Nikolentzos, G., and Vazirgiannis, M.
\newblock Transfer graph neural networks for pandemic forecasting.
\newblock In \emph{AAAI}, volume~35, pp.\  4838--4845, 2021.

\bibitem[Pedregosa(2016)]{pedregosa2016hyperparameter}
Pedregosa, F.
\newblock Hyperparameter optimization with approximate gradient.
\newblock In \emph{ICML}, pp.\  737--746. PMLR, 2016.

\bibitem[Pei et~al.(2020)Pei, Wei, Chang, Lei, and Yang]{pei2020geom}
Pei, H., Wei, B., Chang, K. C.-C., Lei, Y., and Yang, B.
\newblock Geom-gcn: Geometric graph convolutional networks.
\newblock \emph{arXiv preprint arXiv:2002.05287}, 2020.

\bibitem[Platonov et~al.(2023)Platonov, Kuznedelev, Babenko, and Prokhorenkova]{platonov2023characterizing}
Platonov, O., Kuznedelev, D., Babenko, A., and Prokhorenkova, L.
\newblock Characterizing graph datasets for node classification: Homophily-heterophily dichotomy and beyond.
\newblock In \emph{NeurIPS}, 2023.

\bibitem[Ramaswamy et~al.(2016)Ramaswamy, Scott, and Tewari]{ramaswamy2016mixture}
Ramaswamy, H., Scott, C., and Tewari, A.
\newblock Mixture proportion estimation via kernel embeddings of distributions.
\newblock In \emph{ICML}, pp.\  2052--2060. PMLR, 2016.

\bibitem[Scott(2015)]{scott2015rate}
Scott, C.
\newblock A rate of convergence for mixture proportion estimation, with application to learning from noisy labels.
\newblock In \emph{Artificial Intelligence and Statistics}, pp.\  838--846. PMLR, 2015.

\bibitem[Scott et~al.(2013)Scott, Blanchard, and Handy]{scott2013classification}
Scott, C., Blanchard, G., and Handy, G.
\newblock Classification with asymmetric label noise: Consistency and maximal denoising.
\newblock In \emph{COLT}, pp.\  489--511. PMLR, 2013.

\bibitem[Veli{\v{c}}kovi{\'c}(2023)]{velivckovic2023everything}
Veli{\v{c}}kovi{\'c}, P.
\newblock Everything is connected: Graph neural networks.
\newblock \emph{Current Opinion in Structural Biology}, 79:\penalty0 102538, 2023.

\bibitem[Veli{\v{c}}kovi{\'c} et~al.(2017)Veli{\v{c}}kovi{\'c}, Cucurull, Casanova, Romero, Lio, and Bengio]{velivckovic2017graph}
Veli{\v{c}}kovi{\'c}, P., Cucurull, G., Casanova, A., Romero, A., Lio, P., and Bengio, Y.
\newblock Graph attention networks.
\newblock \emph{arXiv preprint arXiv:1710.10903}, 2017.

\bibitem[Wang \& Leskovec(2020)Wang and Leskovec]{wang2020unifying}
Wang, H. and Leskovec, J.
\newblock Unifying graph convolutional neural networks and label propagation.
\newblock \emph{arXiv preprint arXiv:2002.06755}, 2020.

\bibitem[Wu et~al.(2021)Wu, Pan, Du, and Zhu]{wu2021learning}
Wu, M., Pan, S., Du, L., and Zhu, X.
\newblock Learning graph neural networks with positive and unlabeled nodes.
\newblock \emph{TKDD}, 15\penalty0 (6):\penalty0 1--25, 2021.

\bibitem[Wu et~al.(2022)Wu, Zhang, Yan, and Wipf]{wu2022handling}
Wu, Q., Zhang, H., Yan, J., and Wipf, D.
\newblock Handling distribution shifts on graphs: An invariance perspective.
\newblock \emph{arXiv preprint arXiv:2202.02466}, 2022.

\bibitem[Wu et~al.(2023)Wu, Xia, Yu, Han, Niu, Sugiyama, and Liu]{wu2023making}
Wu, Y., Xia, X., Yu, J., Han, B., Niu, G., Sugiyama, M., and Liu, T.
\newblock Making binary classification from multiple unlabeled datasets almost free of supervision.
\newblock \emph{arXiv preprint arXiv:2306.07036}, 2023.

\bibitem[Wu et~al.(2024)Wu, Yao, Xia, Yu, Wang, Han, and Liu]{wu2024mitigating}
Wu, Y., Yao, J., Xia, X., Yu, J., Wang, R., Han, B., and Liu, T.
\newblock Mitigating label noise on graph via topological sample selection.
\newblock \emph{arXiv preprint arXiv:2403.01942}, 2024.

\bibitem[Xia et~al.(2023)Xia, Liu, Zhang, Wu, and Liu]{xia2023coreset}
Xia, X., Liu, J., Zhang, S., Wu, Q., and Liu, T.
\newblock Coreset selection with prioritized multiple objectives.
\newblock \emph{arXiv preprint arXiv:2311.08675}, 2023.

\bibitem[Xinrui et~al.(2023)Xinrui, Geng, Li, Chen, et~al.]{xinrui2023beyond}
Xinrui, W., Geng, C., Li, S.-Y., Chen, S., et~al.
\newblock Beyond myopia: Learning from positive and unlabeled data through holistic predictive trends.
\newblock In \emph{NeurIPS}, 2023.

\bibitem[Xu et~al.(2018)Xu, Li, Tian, Sonobe, Kawarabayashi, and Jegelka]{xu2018representation}
Xu, K., Li, C., Tian, Y., Sonobe, T., Kawarabayashi, K.-i., and Jegelka, S.
\newblock Representation learning on graphs with jumping knowledge networks.
\newblock In \emph{ICML}, pp.\  5453--5462. PMLR, 2018.

\bibitem[Yan et~al.(2022)Yan, Hashemi, Swersky, Yang, and Koutra]{yan2022two}
Yan, Y., Hashemi, M., Swersky, K., Yang, Y., and Koutra, D.
\newblock Two sides of the same coin: Heterophily and oversmoothing in graph convolutional neural networks.
\newblock In \emph{ICDM}, pp.\  1287--1292. IEEE, 2022.

\bibitem[Yang et~al.(2023)Yang, Zhang, Yao, and Kwok]{yang2023positive}
Yang, H., Zhang, Y., Yao, Q., and Kwok, J.
\newblock Positive-unlabeled node classification with structure-aware graph learning.
\newblock In \emph{CIKM}, pp.\  4390--4394, 2023.

\bibitem[Yao et~al.(2020)Yao, Liu, Han, Gong, Niu, Sugiyama, and Tao]{yao2020rethinking}
Yao, Y., Liu, T., Han, B., Gong, M., Niu, G., Sugiyama, M., and Tao, D.
\newblock Rethinking class-prior estimation for positive-unlabeled learning.
\newblock \emph{arXiv preprint arXiv:2002.03673}, 2020.

\bibitem[Yong et~al.(2022)Yong, Pi, Zhang, Xia, Gao, Zhou, Liu, and Han]{yong2022holistic}
Yong, L., Pi, R., Zhang, W., Xia, X., Gao, J., Zhou, X., Liu, T., and Han, B.
\newblock A holistic view of label noise transition matrix in deep learning and beyond.
\newblock In \emph{The Eleventh International Conference on Learning Representations}, 2022.

\bibitem[Yoo et~al.(2021)Yoo, Kim, Yoon, Kim, Jang, and Kang]{yoo2021accurate}
Yoo, J., Kim, J., Yoon, H., Kim, G., Jang, C., and Kang, U.
\newblock Accurate graph-based pu learning without class prior.
\newblock In \emph{ICDM}, pp.\  827--836. IEEE, 2021.

\bibitem[Zhang et~al.(2019)Zhang, Song, Huang, Swami, and Chawla]{zhang2019heterogeneous}
Zhang, C., Song, D., Huang, C., Swami, A., and Chawla, N.~V.
\newblock Heterogeneous graph neural network.
\newblock In \emph{KDD}, pp.\  793--803, 2019.

\bibitem[Zhao et~al.(2020)Zhao, Chen, Hu, and Cho]{zhao2020uncertainty}
Zhao, X., Chen, F., Hu, S., and Cho, J.-H.
\newblock Uncertainty aware semi-supervised learning on graph data.
\newblock \emph{NeurIPS}, 33:\penalty0 12827--12836, 2020.

\bibitem[Zhu et~al.(2020)Zhu, Yan, Zhao, Heimann, Akoglu, and Koutra]{zhu2020beyond}
Zhu, J., Yan, Y., Zhao, L., Heimann, M., Akoglu, L., and Koutra, D.
\newblock Beyond homophily in graph neural networks: Current limitations and effective designs.
\newblock \emph{NeurIPS}, 33:\penalty0 7793--7804, 2020.

\bibitem[Zhu et~al.(2021)Zhu, Rossi, Rao, Mai, Lipka, Ahmed, and Koutra]{zhu2021graph}
Zhu, J., Rossi, R.~A., Rao, A., Mai, T., Lipka, N., Ahmed, N.~K., and Koutra, D.
\newblock Graph neural networks with heterophily.
\newblock In \emph{AAAI}, volume~35, pp.\  11168--11176, 2021.

\bibitem[Zhu et~al.(2022)Zhu, Jin, Loveland, Schaub, and Koutra]{zhu2022does}
Zhu, J., Jin, J., Loveland, D., Schaub, M.~T., and Koutra, D.
\newblock How does heterophily impact the robustness of graph neural networks? theoretical connections and practical implications.
\newblock In \emph{KDD}, pp.\  2637--2647, 2022.

\bibitem[Zhu et~al.(2023)Zhu, Fjeldsted, Holland, Landon, Lintereur, and Scott]{zhu2023mixture}
Zhu, Y., Fjeldsted, A., Holland, D., Landon, G., Lintereur, A., and Scott, C.
\newblock Mixture proportion estimation beyond irreducibility.
\newblock \emph{arXiv preprint arXiv:2306.01253}, 2023.

\end{thebibliography}
\bibliographystyle{icml2024}

\newpage

\appendix
\onecolumn
\section{Proof of the Theorem~\ref{irr_condtion}}
\label{app:A.}

Up to now, almost all CPE algorithms assume $\mathbbm{P}_\mathrm{n}$ to be irreducible with respect to $\mathbbm{P}_\mathrm{p}$~\citep{blanchard2010semi,ivanov2020dedpul,jain2016nonparametric}, or stricter conditions like the anchor set assumption~\citep{ramaswamy2016mixture,yao2020rethinking}. In this section, we briefly review the irreducibility assumptions used for existing CPE estimators on i.i.d data. Then we understand the irreducibility assumption on the graph data when estimating the class prior on a graph.

The so-called irreducibility assumption was proposed by~\citep{blanchard2010semi}:
\begin{definition} [Irreducibility] $\mathbbm{P}_\mathrm{n}$ and $\mathbbm{P}_\mathrm{p}$ are said to satisfy the irreducibility assumption if $\mathbbm{P}_\mathrm{n}$ is not a mixture containing $\mathbbm{P}_\mathrm{p}$. That is, there does not exist a decomposition $\mathbbm{P}_\mathrm{n} = (1-\beta)Q + \beta \mathbbm{P}_\mathrm{p}$, where $Q$ is some probability distribution and $0 < \beta \leq 1$.

Equivalently, the irreducibility assumption assumes the support of $\mathbbm{P}_\mathrm{p}$ is hardly contained in the support of $\mathbbm{P}_\mathrm{n}$ and implies the following fact~\citep{scott2013classification,zhu2023mixture}:
\begin{lemma} \label{irr_ass_variation}
We say that $\mathbbm{P}_\mathrm{n}$ is irreducible with respect to $\mathbbm{P}_\mathrm{p}$ if the infimum of the likelihood ratio
\begin{equation}
    \essinf_{i \in \mathcal{V}} \frac{\mathbbm{P}(\bm{x}_i,\mathcal{G}_{\bm{x}_i}^{\bm{A}}|y=-1)}{\mathbbm{P}(\bm{x}_i,\mathcal{G}_{\bm{x}_i}^{\bm{A}}|y=+1)} = 0.
\end{equation}
\end{lemma}
\end{definition}

Now, we consider a way of understanding irreducibility in terms of a latent graph label model on a graph. The conditional probability of $y$ given $(\bm{x}, \mathcal{G}_{\bm{x}}^{\bm{A}})$ to be defined via
\begin{equation}
    \begin{split}
   \mathbbm{P}(y = +1  | \bm{x},  \mathcal{G}_{\bm{x}}^{\bm{A}}) 
          = \begin{cases}
        \frac{\pi_\mathrm{p}\mathbbm{P}(\bm{x},\mathcal{G}_{\bm{x}}^{\bm{A}}|y=+1)}{\mathbbm{P}(\bm{x},\mathcal{G}_{\bm{x}}^{\bm{A}})}, & \mathbbm{P}(\bm{x},\mathcal{G}_{\bm{x}}^{\bm{A}}) > 0, \\
        0, & \text{otherwise}.
        \end{cases}
    \end{split}
    \label{eq:cond}
\end{equation}

This latent graph label model is inspired by the commonly used label model in the PU learning literature~\cite{bekker2020learning}. $y$ may be viewed as a label indicating which component an observation from $\mathbbm{P}(\bm{x},\mathcal{G}_{\bm{x}}^{\bm{A}})$ was drawn from. Going forward, we use this latent graph label model in addition to the CPE.

\begin{proposition} \label{latent label with CPE} Under the latent label model, the essential supremum 
\begin{equation}
    \begin{split}
    \esssup_{i \in \mathcal{V}} \mathbbm{P}(y = +1 | \bm{x}_i,  \mathcal{G}_{\bm{x}_i}^{\bm{A}})  =     \frac{\pi_\mathrm{p}} {\essinf_{i \in \mathcal{V}} \frac{\mathbbm{P}(\bm{x}_i,\mathcal{G}_{\bm{x}_i}^{\bm{A}})}{\mathbbm{P}(\bm{x}_i,\mathcal{G}_{\bm{x}_i}^{\bm{A}}|y=+1)}}  \\
    = \frac{\pi_\mathrm{p}} {\pi_\mathrm{p} + (1-\pi_\mathrm{p})\essinf_{i \in \mathcal{V}} \frac{\mathbbm{P}(\bm{x}_i,\mathcal{G}_{\bm{x}_i}^{\bm{A}}|y=-1)}{\mathbbm{P}(\bm{x}_i,\mathcal{G}_{\bm{x}_i}^{\bm{A}}|y=+1)}}.
    \end{split}
\end{equation}
\end{proposition}

Combining Proposition~\ref{latent label with CPE} and Lemma~\ref{irr_ass_variation}, we conclude the irreducible condition for satisfying the irreducibility assumption on a graph as in the Theorem~\ref{irr_condtion}.

\section{Proof of the Theorem~\ref{heter_pro_LPA}}
\label{app:B.}

Inspired from the~\citep{koh2017understanding,wang2020unifying,xu2018representation}, we can propose

\begin{lemma}
\label{lemma:2}
    Let $\mathcal U_j^{a \rightarrow b}$ be a path $[\bm{x}^{(j)}, \bm{x}^{(j-1)},\cdots, \bm{x}^{(0)}]$ of length $j$ from positive node $(\bm{x}_a,\mathcal{G}_{\bm{x}_a}^{\bm{A}})$ to negative node $(\bm{x}_b,\mathcal{G}_{\bm{x}_b}^{\bm{A}})$, where $\bm{x}^{(j)} = \bm{x}_a$, $\bm{x}^{(0)} = \bm{x}_b$, $\bm{x}^{(i-1)} \in \mathcal N(\bm{x}^{(i)})$ for $i = j,\cdots, 1$, and all nodes along the path are unlabeled except $\bm{x}^{(0)}$.
    Then we have
    \begin{equation}
        {HI((\bm{x}_a,\mathcal{G}_{\bm{x}_a}^{\bm{A}}),(\bm{x}_b,\mathcal{G}_{\bm{x}_b}^{\bm{A}});k)}  = \sum_{j=1}^k \sum_{\mathcal U_j^{a \rightarrow b}} \prod_{i=j}^1 \tilde a_{\bm{x}^{(i-1)}, \bm{x}^{(i)}},
    \end{equation}
    where $\tilde a_{\bm{x}^{(i-1)}, \bm{x}^{(i)}}$ is the normalized weight of edge $(\bm{x}^{(i)}, \bm{x}^{(i-1)})$.
\end{lemma}

\begin{proof}
When running LPA, the value of the negative class-posterior probability in $y_a^{(\cdot)}$ (denoted by ${\mathbbm{P}(y_a = -1 | \bm{x}_a,\mathcal{G}_{\bm{x}_a}^{\bm{A}})}^{(\cdot)}$) mainly comes from the nodes with initial negative label.
It is clear that
\begin{equation}
    \left| {\mathbbm{P}(y_a = -1 | \bm{x}_a,\mathcal{G}_{\bm{x}_a}^{\bm{A}})}^{(k)} -{\mathbbm{P}(y_a = -1 | \bm{x}_a,\mathcal{G}_{\bm{x}_a}^{\bm{A}})}^{(0)} \right|:= \sum_{\bm x_b: y_b = -1} \sum_{j=1}^k \sum_{\mathcal U_j^{a \rightarrow b}} \prod_{i=j}^1 \tilde a_{\bm{x}^{(i-1)}, \bm{x}^{(i)}},
\end{equation}
which equals $\sum_{\bm{x}_b: y_b = -1}   {HI((\bm{x}_a,\mathcal{G}_{\bm{x}_a}^{\bm{A}}),(\bm{x}_b,\mathcal{G}_{\bm{x}_b}^{\bm{A}});k)}$ according to Lemma \ref{lemma:2}.
Therefore, we have
\begin{equation}
    \left| {\mathbbm{P}(y_a = -1 | \bm{x}_a,\mathcal{G}_{\bm{x}_a}^{\bm{A}})}^{(k)} -{\mathbbm{P}(y_a = -1 | \bm{x}_a,\mathcal{G}_{\bm{x}_a}^{\bm{A}})}^{(0)} \right| := \sum_{\bm{x}_b: y_b = -1} {HI((\bm{x}_a,\mathcal{G}_{\bm{x}_a}^{\bm{A}}),(\bm{x}_b,\mathcal{G}_{\bm{x}_b}^{\bm{A}});k)},
\end{equation}

which is the same as 
\begin{equation}
    \left| {\mathbbm{P}(y_a = -1 | \bm{x}_a,\mathcal{G}_{\bm{x}_a}^{\bm{A}})}^{(k)} -{\mathbbm{P}(y_a = -1 | \bm{x}_a,\mathcal{G}_{\bm{x}_a}^{\bm{A}})}^{(0)} \right| := \sum_{\bm{x}_i \in \mathcal{V}, \bm{x}_i \neq \bm{x}_a} {HI((\bm{x}_a,\mathcal{G}_{\bm{x}_a}^{\bm{A}}),(\bm{x}_i,\mathcal{G}_{\bm{x}_i}^{\bm{A}});k)} 
\end{equation}

\end{proof}

\section{Convergence of GPL}
\label{app:c.}
The convergence of the bilevel optimization~\citep{xia2023coreset} using approximated solution of inner loop was established in~\citep{pedregosa2016hyperparameter,yong2022holistic}. We restate it here for completeness.
\begin{theorem} [Convergence, Theorem 3.3 of~\citep{pedregosa2016hyperparameter}] Suppose $\mathcal{L}_{\textit{GNN}}(\hat{\pi}_\mathrm{p}, w,\hat{\bm{A}})$ is smooth w.r.t. $w$, $\mathcal{L}_{\textit{LPA}}(\bm{A})^{(k)}$ is $\beta$-smooth and $\alpha$-strongly convex w.r.t. $\bm{A}$. We solve the inner loop by unrolling
$J$ steps, choose the learning rate in the outer loop as $\eta_{\tau} = 1/\sqrt{\tau}$, then we arrive at an approximately stationary point as follows after $R$ steps:
\begin{equation}
    \mathbbm{E} \left[ \sum_{\tau = 1}^{R} \frac{\eta_{\tau} \left\|\nabla_{w} \mathcal{L}_{\textit{GNN}}(\hat{\pi}_\mathrm{p}, w,\hat{\bm{A}}) \right\|_{2}^{2}}{\sum_{\tau = 1}^{R}\eta_{\tau}}\right] \leq \tilde{O} \left( \epsilon + \frac{\epsilon^{2}+1}{\sqrt{R}}\right),
\end{equation}
where $\tilde{O}$ absorbs constants and logarithmic terms and $\epsilon = (1 - \alpha/\beta)^{J}$.
\end{theorem}

\section{Proof of the Theorem~\ref{est_error}}
\label{proff_est_error}

According to the Lemma~\ref{lemma_err}, the upper bound of $e^{\bm{A}}$ is
\begin{equation}
\label{bound_1}
    b^{\bm{A}} = \frac{c}{\mathbbm{Q}_\mathrm{p}^{\bm{A}}(c^*)}\left( \sqrt{\frac{\log(4/\delta)}{n_\mathrm{u}}} + \sqrt{\frac{\log(4/\delta)}{n_\mathrm{p}}}\right)
\end{equation}
and the upper bound of $e^{\hat{\bm{A}}}$ is
\begin{equation}
\label{bound_2}
    b^{\hat{\bm{A}}} = \frac{c}{\mathbbm{Q}_\mathrm{p}^{\hat{\bm{A}}}(c^*)}\left( \sqrt{\frac{\log(4/\delta)}{n_\mathrm{u}}} + \sqrt{\frac{\log(4/\delta)}{n_\mathrm{p}}}\right).
\end{equation}

The
$\mathbbm{Q}^{\bm{A}}_{\mathrm{p}}(c^*)  := \sum_{\bm{x}_i \in \mathcal{P}}  \mathbbm{1} [f_{w}(\bm{x}_i,\mathcal{G}_{\bm{x}_i}^{\bm{A}}) \geq c^*] / n$ and $\mathbbm{Q}^{\hat{\bm{A}}}_{\mathrm{p}}(c^*)  := \sum_{\bm{x}_i \in \mathcal{P}}  \mathbbm{1} [f_{w}(\bm{x}_i,\mathcal{G}_{\bm{x}_i}^{\hat{\bm{A}}}) \geq c^*] / n$, we know that the GNN model transforms each node $(\bm{x},\mathcal{G}_{\bm{x}}^{\bm{A}})$ to a positive posterior probability $\bm{z} \in  [0,1]$, i.e., $\bm{z} = f_{w}(\bm{x},\mathcal{G}_{\bm{x}}^{\bm{A}})$. Numerous studies ~\citep{yan2022two,zhu2020beyond,zhu2021graph} have both theoretically and empirically demonstrated that heterophilic structures impede the confidence of GNN predictions. Therefore, for positive nodes, the positive posterior probability $\hat{\bm{z}} = f_{w}(\bm{x},\mathcal{G}_{\bm{x}}^{\hat{\bm{A}}})$ is greater than  $\bm{z} = f_{w}(\bm{x},\mathcal{G}_{\bm{x}}^{\bm{A}})$ due to the reduction of heterophilic structure.

Thus, we have
\begin{equation}
    \sum_{\bm{x}_i \in \mathcal{P}}  \mathbbm{1} [f_{w}(\bm{x}_i,\mathcal{G}_{\bm{x}_i}^{\hat{\bm{A}}}) \geq c^*] / n \geq \sum_{\bm{x}_i \in \mathcal{P}}  \mathbbm{1} [f_{w}(\bm{x}_i,\mathcal{G}_{\bm{x}_i}^{\bm{A}}) \geq c^*] / n,
\end{equation}
then
\begin{equation}
    \mathbbm{Q}^{\hat{\bm{A}}}_{\mathrm{p}}(c^*)  \geq  \mathbbm{Q}^{\bm{A}}_{\mathrm{p}}(c^*).
\end{equation}

According to Equation~\ref{bound_1} and~\ref{bound_2}, we have following inequality:
\begin{equation}
    b^{\hat{\bm{A}}} \leq b^{\bm{A}}.
\end{equation}

\section{Proof of Theorem \ref{thm:decrease}}
	\label{app:d}

Firstly, we analyze the layer-wise propagation rules of GNN for the node $\bm{x}_i$:
\begin{gather*}
    \bm{x}^{(l+1)}_i = \sigma \left( \sum_{\bm{x}_j \in \mathcal N(\bm{x}_i)}\tilde{\bm{A}}_{ij} \bm{x}_i^{(l)}w^{(l)} \right),
\end{gather*}
where $\tilde{\bm{A}}_{ij} = {\bm{A}}_{ij} / d_{ii}$ is the normalized weight of edge $e_{ij}$ with the degree $d_{ii}$. This formula can be decomposed into the following two steps:
(1) In \textit{aggregation} step, we calculate the aggregated representation ${\bm h}_i^{(l)}$ of all neighborhoods $\mathcal N(\bm{x}_i)$:
\begin{gather*}
    {\bm h}_i^{(l)} = \sum_{\bm{x}_j \in \mathcal N(\bm{x}_i)} \tilde{\bm{A}}_{ij} {\bm{x}}_j^{(l)}.
\end{gather*}
(2) In \textit{transformation} step, the aggregated representation ${\bm h}_i^{(l)}$ is mapped to a new space by a transformation matrix and nonlinear function $\sigma$:
\begin{gather*}
    {\bm{x}}_i^{(l+1)} = \sigma \big( {\bm h}_i^{(l)} w^{(l)} \big).
\end{gather*}

we assume that the dimension of node representations is one, but note that the conclusion can be easily generalized to the case of multi-dimensional representations since the function $D_{\mathrm{P}\mathrm{N}}({\bm{x}})$ can be decomposed into the sum of one-dimensional cases. This proof is according to~\citep{wang2020unifying}.
		In the following of this proof,  we still use bold notations ${\bm{x}}_i^{(l)}$ and ${\bm h}_i^{(l)}$ to denote node representations, but keep in mind that they are scalars rather than vectors.
		
		We give two lemmas before proving Theorem \ref{thm:decrease}.
		The first one is about the gradient of $D_{\mathrm{P}\mathrm{N}}({\bm{x}})$:
		\begin{lemma}
		\label{lemma:3}
			${\bm h}_i^{(l)} = {\bm{x}}_i^{(l)} - \frac{\partial D_{\mathrm{P}\mathrm{N}}({\bm{x}}^{(l)})}{\partial {\bm{x}}_i^{(l)}}$.
		\end{lemma}
	
		\begin{proof}
  \begin{equation}
      {\bm{x}}_i^{(l)} - \frac{\partial D_{\mathrm{P}\mathrm{N}}({\bm{x}}^{(l)})}{\partial {\bm{x}}_i^{(l)}} = {\bm{x}}_i^{(l)} - \sum_{\bm{x}_j \in \mathcal N(\bm{x}_i)} \tilde{\bm{A}}_{ij} ({\bm{x}}_i^{(l)} - {\bm{x}}_j^{(l)}) = \sum_{\bm{x}_j \in \mathcal N(\bm{x}_i)}  \tilde{\bm{A}}_{ij} {\bm{x}}_j^{(l)} = {\bm h}_i^{(l)}
  \end{equation}
		\end{proof}
		
		It is interesting to see from Lemma \ref{lemma:3} that the aggregation step in GCN is equivalent to running gradient descent for one step with a step size of one.
		However, this is not able to guarantee that $D_{\mathrm{P}\mathrm{N}}({\bm h}^{(l)}) \leq D_{\mathrm{P}\mathrm{N}}({\bm{x}}^{(l)})$ because the step size may be too large to reduce the value of $D_{\mathrm{P}\mathrm{N}}$.
		
		The second lemma is about the Hessian of $D_{\mathrm{P}\mathrm{N}}({\bm{x}})$:
		
		\begin{lemma}
		\label{lemma:4}
			$\nabla^2 D_{\mathrm{P}\mathrm{N}}({\bm{x}}) \preceq 2I$, or equivalently, $2I - \nabla^2 D_{\mathrm{P}\mathrm{N}}({\bm{x}})$ is a positive semidefinite matrix.
		\end{lemma}
		
		\begin{proof}
			We first calculate the Hessian of $D_{\mathrm{P}\mathrm{N}}({\bm{x}}) = \frac{1}{2} \sum_{\bm{x}_i \in \mathrm{P}, \bm{x}_j \in \mathrm{N}}\tilde a_{ij} \| {\bm{x}}_i - {\bm{x}}_j \|_2^2$:
			\begin{equation}
				\nabla^2 D_{\mathrm{P}\mathrm{N}}({\bm{x}}) =
				\left[
					\begin{matrix}
						1-\tilde a_{11} & -\tilde a_{12} & \cdots & -\tilde a_{1n} \\
						-\tilde a_{21} & 1-\tilde a_{22} & \cdots & -\tilde a_{2n} \\
						\vdots & \vdots & \ddots & \vdots \\
						-\tilde a_{n1} & -\tilde a_{n2} & \cdots & 1-\tilde a_{nn} \\
					\end{matrix}
				\right]
				= I - D_{\mathrm{P}\mathrm{N}}^{-1}A.
			\end{equation}
			Therefore, $2I - \nabla^2 D_{\mathrm{P}\mathrm{N}}({\bm{x}}) = I + D_{\mathrm{P}\mathrm{N}}^{-1}\bm A$.
			Since $D_{\mathrm{P}\mathrm{N}}^{-1} \bm A$ is Markov matrix (i.e., each entry is non-negative and the sum of each row is one), its eigenvalues are within the range [-1, 1], so the eigenvalues of $I + D_{\mathrm{P}\mathrm{N}}^{-1} \bm A$ are within the range [0, 2].
			Therefore, $I + D_{\mathrm{P}\mathrm{N}}^{-1} \bm A$ is a positive semidefinite matrix, and we have $\nabla^2 D_{\mathrm{P}\mathrm{N}}({\bm{x}}) \preceq 2I$.
		\end{proof}
		
		We can now prove Theorem \ref{thm:decrease}:
		
		\begin{proof}
			Since $D_{\mathrm{P}\mathrm{N}}$ is a quadratic function, we perform a second-order Taylor expansion of $D_{\mathrm{P}\mathrm{N}}$ around ${\bm{x}}^{{(l)}}$ and obtain the following inequality:
			\begin{equation}
			\begin{split}
				D_{\mathrm{P}\mathrm{N}}({\bm h}^{(l)}) = & D_{\mathrm{P}\mathrm{N}}({\bm{x}}^{(l)}) + \nabla D_{\mathrm{P}\mathrm{N}}({\bm{x}}^{(l)})^\top ({\bm h}^{(l)} - {\bm{x}}^{(l)}) + \frac{1}{2} ({\bm h}^{(l)} - {\bm{x}}^{(l)})^\top \nabla^2 D_{\mathrm{P}\mathrm{N}}({\bm{x}}) ({\bm h}^{(l)} - {\bm{x}}^{(l)}) \\
				= & D_{\mathrm{P}\mathrm{N}}({\bm{x}}^{(l)}) - \nabla D_{\mathrm{P}\mathrm{N}}({\bm{x}}^{(l)})^\top \nabla D_{\mathrm{P}\mathrm{N}}({\bm{x}}^{(l)}) + \frac{1}{2} \nabla D_{\mathrm{P}\mathrm{N}}({\bm{x}}^{(l)})^\top \nabla^2 D_{\mathrm{P}\mathrm{N}}({\bm{x}}) \nabla D_{\mathrm{P}\mathrm{N}}({\bm{x}}^{(l)}) \\
				\leq & D_{\mathrm{P}\mathrm{N}}({\bm{x}}^{(l)}) - \nabla D_{\mathrm{P}\mathrm{N}}({\bm{x}}^{(l)})^\top \nabla D_{\mathrm{P}\mathrm{N}}({\bm{x}}^{(l)}) + \nabla D_{\mathrm{P}\mathrm{N}}({\bm{x}}^{(l)})^\top \nabla D_{\mathrm{P}\mathrm{N}}({\bm{x}}^{(l)}) = D_{\mathrm{P}\mathrm{N}}({\bm{x}}^{(l)}).
			\end{split}
			\end{equation}
		\end{proof}

\section{Experiment}
\label{exper_de}

\subsection{Datasets}
The statistical information of used datasets is shown in Table~\ref{table::dataset}.
\begin{table}[ht]
\begin{center}
\caption{Summary of datasets.}
\label{table::dataset}
\begin{tabular}{c|ccc|cc}
\toprule
\multicolumn{1}{c|}{\bf Name} & \multicolumn{1}{c}{\bf Nodes} & \multicolumn{1}{c}{\bf Edges} & \multicolumn{1}{c|}{\bf Features} & \multicolumn{1}{c}{\bf Pos.} & \multicolumn{1}{c}{\bf Neg.}\\
\hline
Cora     & $2,708$  & $5,278$   & $1,433$ & $818$   & $1,890$ \\
CiteSeer & $3,327$  & $4,552$   & $3,703$ & $701$   & $2,626$ \\
PubMed   & $19,717$ & $44,324$  & $500$   & $7,875$ & $11,842$\\
WikiCS   & $11,701$ & $215,603$ & $300$   & $2,679$ & $9,022$ \\
\hline
Cornell   & $183$ & $298$ & $1,703$   & $41$ & $142$ \\
Chameleon   & $2,277$ & $36,101$ & $2,325$   & $260$ & $2,017$ \\
Squirrel   & $5,201$ & $2,089$ & $217,073$   & $521$ & $4,680$ \\
Actor   & $7,600$ & $30,019$ & $932$   & $982$ & $6,618$ \\
Wisconsin    & $251$ & $515$ & $1,703$   & $59$ & $192$ \\
Texas & $183$ & $325$ & $1,703$   & $50$ & $133$ \\

\bottomrule
\end{tabular}
\end{center}
\end{table}

\subsection{Baseline Details} \label{baseline}
In more detail, we employ baselines:
\begin{itemize}
\item [(1)] GCN: Utilizing the cross-entropy loss function to train a Graph Convolutional Network (GCN) and treating all unlabeled nodes as negative examples.

\item [(2)] MLP: Applying the cross-entropy loss during the training of a Multi-Layer Perceptron (MLP), where unlabeled nodes are treated as negative examples.

\item [(3)] GCN+TED: Combining the GCN model with the TED method~\citep{garg2021mixture}, a state-of-the-art approach specifically tailored for PU learning, which estimates the unknown class prior.

\item [(4)] MLP+TED: Integrating the MLP model with the TED method for PU learning.

\item [(5)]  GCN+NNPU: Adopting a GCN model with the Non-Negative PU (NNPU) method~\citep{kiryo2017positive}, which employs a non-negative risk estimator specially designed for PU learning to estimate the unknown negative risk.

\item [(6)]  MLP+NNPU: Using an MLP model with the NNPU method.

\item [(7)]  LSDAN~\citep{ma2017pu}: Improving
GCN on PU learning by the long-short distance attention that effectively
combines the information of multi-hop neighbors.

\item [(8)] GRAB~\citep{yoo2021accurate}: Estimating class priors using heuristics and learning from positively labeled nodes in the presence of unlabeled nodes, leveraging the homophily graph structure.

\item [(9)] PU-GNN~\citep{yang2023positive}: Introducing a distance-aware loss that leverages homophily within graphs to provide more precise supervision in the context of graph-based Positive-Unlabeled (PU) learning.

\end{itemize}

\subsection{Performance with different GNN architectures}
We evaluate our proposed GPL on different GNN architectures, i.e., GCN~\citep{kipf2016semi}, GAT~\citep{velivckovic2017graph}, ARMA~\citep{bianchi2021graph} and APPNP~\citep{gasteiger2018predict}. The experiments are conducted on various datasets, which are shown in Table~\ref{tab:GNN_architectures}. As can be seen, GPL performs similarly on different GNN architectures, showing consistent generalization on different architectures.

\begin{table*} [h]

    \centering
    \caption{
    Mean F1 score $\pm$ stdev over different datasets. The experimental results are reported over five trials. Bold numbers are superior results.
    }
    \vspace{-5pt}
    \label{tab:GNN_architectures}
    \resizebox{\linewidth}{!}{
    \begin{tabular}{l|c|c|c|c| c|c|c|c|c|c} 
    \toprule
     \texttt{\bf Backbone}   &  \texttt{\bf Cora}           &   \texttt{\bf Pubmed}           &   \texttt{\bf Citeseer}   &   \texttt{\bf Wiki-CS}             &   \texttt{\bf Cornell} & \texttt{\bf Chameleon}        &  \texttt{\bf Squirrel}   &   \texttt{\bf Actor}           &   \texttt{\bf Wisconsin}            &   \texttt{\bf Texas}  \\

           \midrule
        GPL+GCN & $ {81.9{\scriptstyle\pm0.5}}$ 
        &$ {79.1{\scriptstyle\pm0.4}}$
        &$ {74.1{\scriptstyle\pm0.9}}$
        &$ \bf{82.3{\scriptstyle\pm0.8}}$
        &$ {37.9{\scriptstyle\pm1.3}}$
        &$ {36.2{\scriptstyle\pm1.4}}$
        &$ {37.7{\scriptstyle\pm5.8}}$
        &$ {48.0{\scriptstyle\pm4.2}}$
        &$ \bf{45.8{\scriptstyle\pm2.7}}$
        &$ {46.3{\scriptstyle\pm3.2}}$
\\
           \midrule
        GPL+GAT & $ \bf{82.1{\scriptstyle\pm0.5}}$ 
        &$ \bf{79.8{\scriptstyle\pm0.7}}$
        &$ {74.7{\scriptstyle\pm0.2}}$
        &$ {81.9{\scriptstyle\pm0.3}}$
        &$ {36.3{\scriptstyle\pm1.1}}$
        &$ {35.7{\scriptstyle\pm0.9}}$
        &$ \bf{37.9{\scriptstyle\pm4.6}}$
        &$ {48.8{\scriptstyle\pm3.5}}$
        &$ {44.8{\scriptstyle\pm1.8}}$
        &$ \bf{46.6{\scriptstyle\pm2.1}}$
\\

           \midrule
        GPL+ARMA & $ {81.2\scriptstyle\pm1.1}$ 
        &$ {78.9{\scriptstyle\pm0.9}}$
        &$ {74.1{\scriptstyle\pm0.5}}$
        &$ {81.7{\scriptstyle\pm0.7}}$
        &$ \bf{38.0{\scriptstyle\pm0.8}}$
        &$ {35.9{\scriptstyle\pm1.2}}$
        &$ {36.9{\scriptstyle\pm4.5}}$
        &$ {47.0{\scriptstyle\pm2.2}}$
        &$ {44.9{\scriptstyle\pm2.3}}$
        &$ {46.5{\scriptstyle\pm2.9}}$
\\
           \midrule
        GPL+APPNP & $ {81.5{\scriptstyle\pm0.6}}$ 
        &$ {79.3{\scriptstyle\pm0.5}}$
        &$ \bf{74.9{\scriptstyle\pm0.7}}$
        &$ {82.1{\scriptstyle\pm0.6}}$
        &$ {37.7{\scriptstyle\pm1.8}}$
        &$ \bf{36.8{\scriptstyle\pm1.1}}$
        &$ {37.1{\scriptstyle\pm4.9}}$
        &$ \bf{49.5{\scriptstyle\pm3.8}}$
        &$ {45.1{\scriptstyle\pm2.9}}$
        &${45.9{\scriptstyle\pm2.6}}$
\\

	   \bottomrule
    \end{tabular}
    }
\end{table*}

\subsection{Hyperparameter sensitivity}
In GPL, the hyperparameter $K$ affects the performance by controlling the iteration times in the label propagation loss. To assess the sensitivity of GPL to $K$, we analyzed its effect as depicted in Figure~\ref{fig:hyper}. Our findings reveal an initial increase in the F1 score with a rising $K$, indicative of effective label propagation loss (LPL) functioning. However, a significant drop in the F1 score is observed when $K$ increases excessively, suggesting that overly high iteration counts hinder the convergence of LPL and performance. 

\begin{figure*}[h]
  \centering 

{
    \begin{minipage}[b]{0.47\linewidth}
      \centering
      \includegraphics[width=\linewidth]{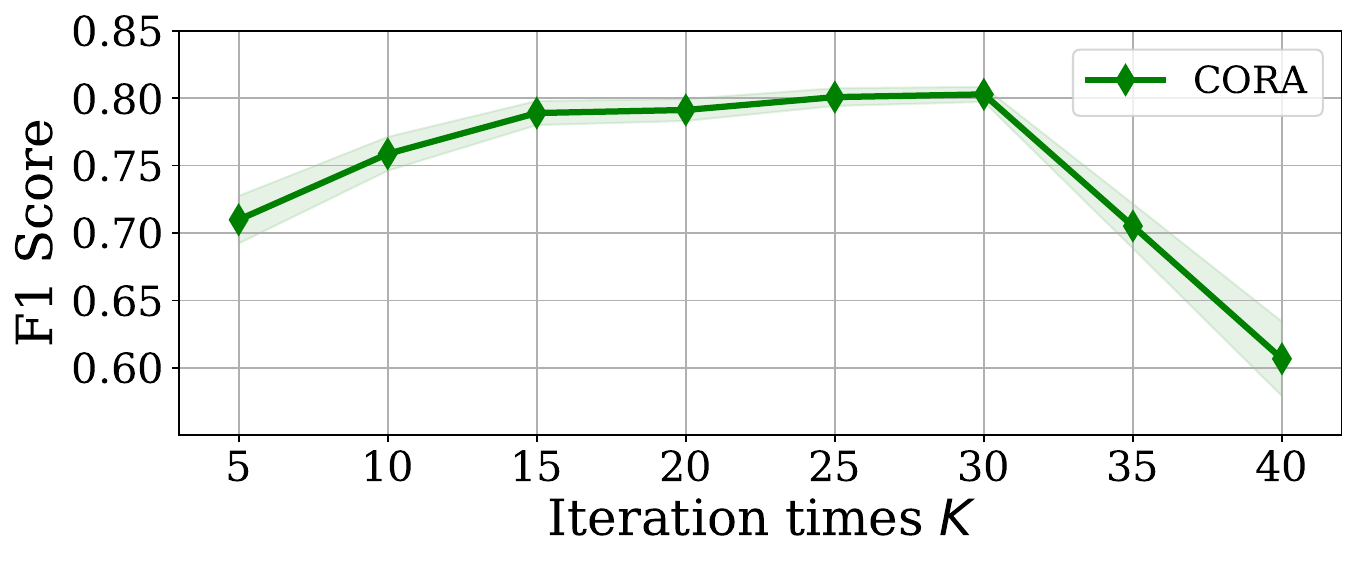}
      \includegraphics[width=\linewidth]{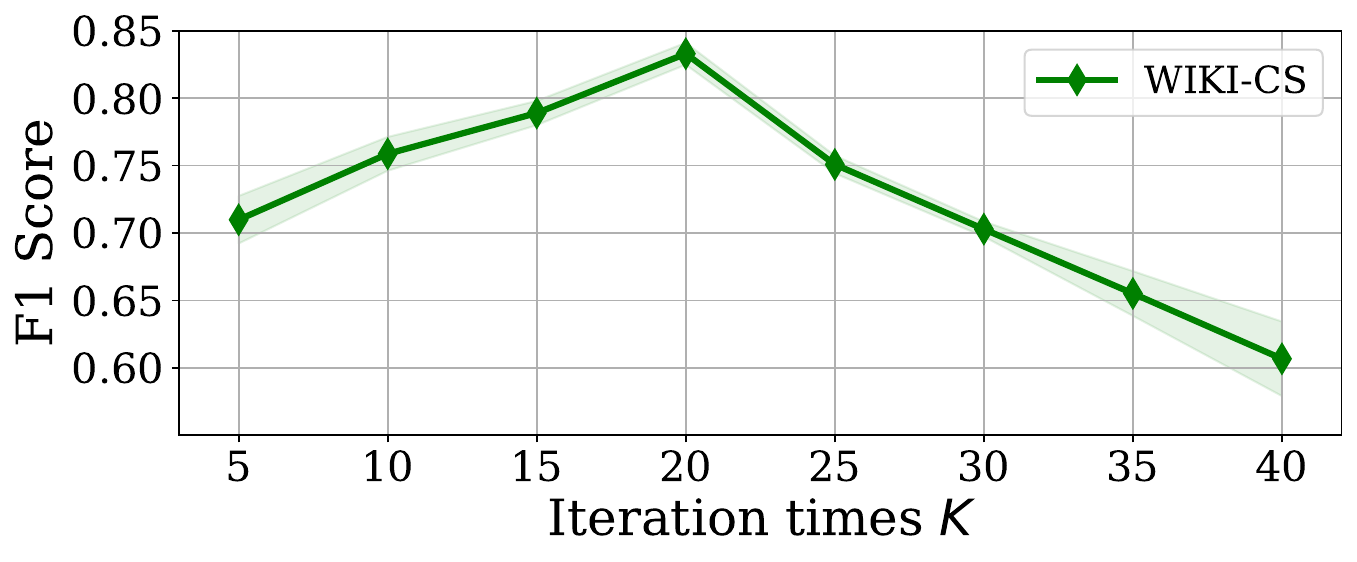}
    \end{minipage}
  }
\hspace{-2mm}
{
    \begin{minipage}[b]{0.47\linewidth}
      \centering
      \includegraphics[width=\linewidth]{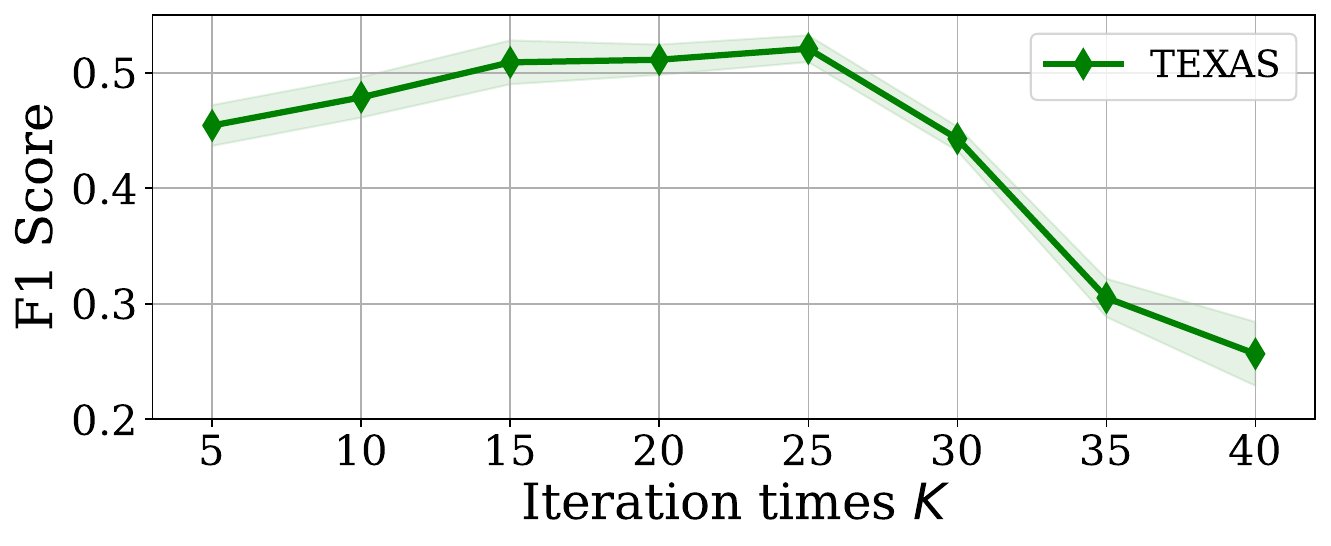}
      \includegraphics[width=\linewidth]{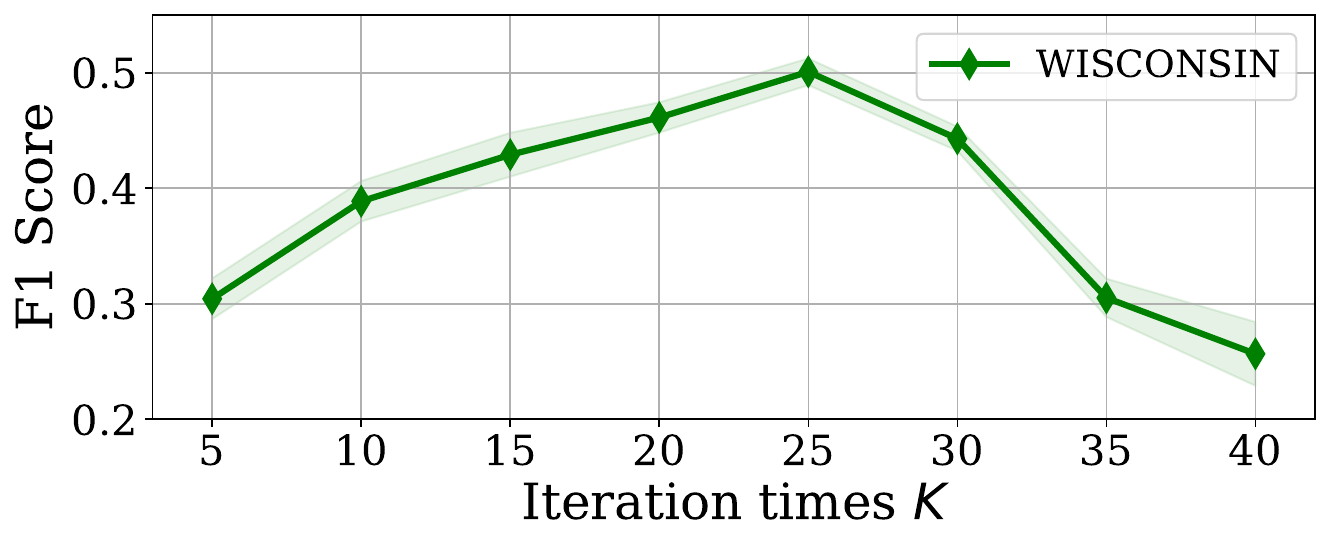}
    \end{minipage}
  }

  \vfill
  \caption{The F1 score of GPL with increasing $K$ on \textit{Cora}, \textit{WIKI-CS}, \textit{Texas} and \textit{Wisconsin}.}
  \label{fig:hyper}
\end{figure*}

\end{document}